\documentclass[sigconf]{acmart}

\AtBeginDocument{%
  }

\copyrightyear{2025}
\acmYear{2025}
\setcopyright{rightsretained}
\acmConference[KDD '25]{Proceedings of the 31st ACM SIGKDD Conference on Knowledge Discovery and Data Mining V.1}{August 3--7, 2025}{Toronto, ON, Canada}
\acmBooktitle{Proceedings of the 31st ACM SIGKDD Conference on Knowledge Discovery and Data Mining V.1 (KDD '25), August 3--7, 2025, Toronto, ON, Canada}
\acmDOI{10.1145/3690624.3709196}
\acmISBN{979-8-4007-1245-6/25/08}

\makeatletter
\gdef\@copyrightpermission{
 \begin{minipage}{0.3\columnwidth}
\href{https://creativecommons.org/licenses/by/4.0/}{\includegraphics[width=0.90\textwidth]{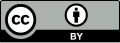}}
 \end{minipage}\hfill
 \begin{minipage}{0.7\columnwidth}
\href{https://creativecommons.org/licenses/by/4.0/}{This work is licensed under a Creative Commons
Attribution International 4.0 License.}
 \end{minipage}
 \vspace{5pt}
}
\makeatother

\usepackage[utf8]{inputenc} %
\usepackage[T1]{fontenc}    %

\usepackage{url}            %
\usepackage{booktabs}       %
\usepackage{amsfonts}       %
\usepackage{nicefrac}       %
\usepackage{microtype}      %
\usepackage{xcolor}         %

\newcommand{\alg}{{{ResMoE}}}

\usepackage{amsmath}
\usepackage{amssymb}
\usepackage{mathtools}
\usepackage{amsthm}
\usepackage{setspace}
\theoremstyle{plain}

\theoremstyle{definition}

\theoremstyle{remark}

\usepackage{paperpkg}
\usepackage{algorithm}
\usepackage{algpseudocode}

\usepackage{wrapfig}

\usepackage[capitalize,noabbrev]{cleveref}

\begin{document}
\title{ResMoE: Space-efficient Compression of Mixture of Experts LLMs via Residual Restoration}

\author{Mengting Ai}
\authornote{Mengting, Tianxin, and Yifan contributed equally to this work.}
\affiliation{
  \institution{
UIUC}
  \city{Champaign}
  \state{IL}
  \country{USA}
}
\email{mai10@illinois.edu}

\author{Tianxin Wei}
\authornotemark[1]
\affiliation{
  \institution{
UIUC}
  \city{Champaign}
  \state{IL}
  \country{USA}
}
\email{twei10@illinois.edu}

\author{Yifan Chen}
\authornotemark[1]
\authornote{Correspondence to: Yifan Chen and Jingrui He.}
\affiliation{
  \institution{
HKBU}
  \city{Hong Kong}
  \country{CHN}
}
\email{yifanc@hkbu.edu.hk}

\author{Zhichen Zeng}
\affiliation{
  \institution{
UIUC}
  \city{Champaign}
  \state{IL}
  \country{USA}
}
\email{zhichenz@illinois.edu}

\author{Ritchie Zhao}
\affiliation{
  \institution{
NVIDIA}
  \city{Redmond}
  \state{WA}
  \country{USA}
}
\email{rz252@cornell.edu}

\author{Girish Varatkar}
\affiliation{
  \institution{
Apple}
  \city{Cupertino}
  \state{CA}
  \country{USA}
}
\email{girish_v_varatkar@apple.com}

\author{Bita Darvish Rouhani}
\affiliation{
  \institution{
NVIDIA}
  \country{USA}
}
\email{brouhani@nvidia.com}

\author{Xianfeng Tang}
\affiliation{
  \institution{
Amazon}
  \city{Palo Alto}
  \state{CA}
  \country{USA}
}
\email{xianft@amazon.com}

\author{Hanghang Tong}
\affiliation{
  \institution{
UIUC}
  \city{Champaign}
  \state{IL}
  \country{USA}
}
\email{htong@illinois.edu}

\author{Jingrui He}
\authornotemark[2]
\affiliation{
  \institution{
UIUC}
  \city{Champaign}
  \state{IL}
  \country{USA}
}
\email{jingrui@illinois.edu}

\renewcommand{\shortauthors}{Mengting Ai et al.}

\begin{abstract}

Mixture-of-Experts (MoE) Transformer, the backbone architecture of multiple phenomenal language models, 
leverages sparsity by activating only a fraction of model parameters for each input token. The sparse structure, while allowing constant time costs, results in space inefficiency: we still need to load all the model parameters during inference.
We introduce \alg, an innovative MoE approximation framework that utilizes Wasserstein barycenter to extract a common expert (barycenter expert) and approximate
the residuals between this barycenter expert and the original ones. 
\alg~enhances the space efficiency for inference of large-scale MoE Transformers in a one-shot and data-agnostic manner without retraining while maintaining minimal accuracy loss, thereby paving the way for broader accessibility to large language models. We demonstrate the effectiveness of \alg~through extensive experiments on Switch Transformer, Mixtral, and DeepSeekMoE models. The results show that \alg~can reduce the number of parameters in an expert by up to 75\% while maintaining comparable performance. 
The code is available at \url{https://github.com/iDEA-iSAIL-Lab-UIUC/ResMoE}.

\end{abstract}


\maketitle
\section{Introduction}

The profound impact of the Transformer architecture in the domain of machine learning is undeniable, for the fields including natural language processing 
~\citep{NIPS2017_3f5ee243,devlin2019bert,fedus2022switch,anil2023palm,OpenAI22,raffel2023exploring} and computer vision ~\citep{liu2021swin,Wang_2021_ICCV,Fan_2021_ICCV}, to name a few. To further improve the capabilities of pre-trained large language models (LLMs), one general strategy is to scale up their parameters. Mixture-of-Experts (MoE) \citep{shazeer2017outrageously} extends the traditional feedforward neural network (FFN) layer by replacing a single multilayer perceptron (MLP) with multiple MLPs, referred to as ``experts''. While enhancing the performance, sparse MoE keeps computing costs (FLOPs) comparable to the original dense model, as only a few selected experts will be activated each time. The framework of an MoE layer is demonstrated in Fig.~\ref{fig:moe}. 
Specifically, the input token $x$ is passed to the router gate network, returning the sparse and normalized top-$k$ scores used to activate the following experts. Only experts with a score larger than 0 will be activated, and the continued results will then be calculated through those activated expert MLPs. 
The output $y$ will then be obtained through a weighted sum of each activated expert's output $y_k$.
Switch Transformer \citep{fedus2022switch} 
exemplifies this approach by expanding the T5 model \citep{raffel2023exploring} to an MoE structure, scaling it up to at most 2,048 times the size of the original dense T5 model. Similarly, Mixtral~\citep{jiang2024mixtral} upscales Mistral 7B~\citep{jiang2023mistral} to an 8$\times$7B MoE structure, achieving performance that matches or even surpasses that of Llama2 70B~\citep{touvron2023llama}. DeepSeekMoE~\citep{dai2024deepseekmoe} utilizes fine-grained experts compared to the other structures, with 64 experts per layer.

\begin{figure}[t]
    \centering
    \includegraphics[width=1.0\columnwidth]{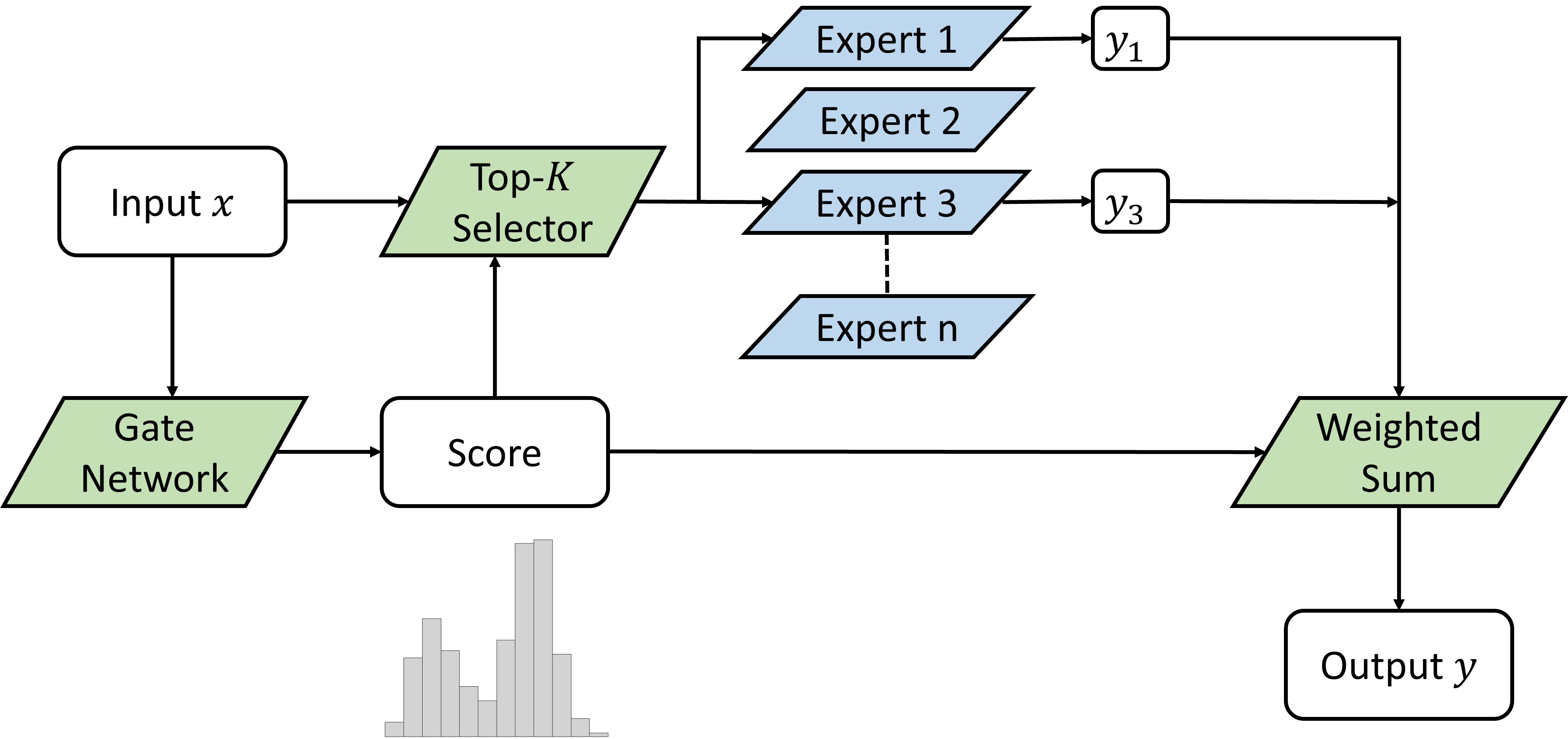}
    \caption{In this illustrative example of MoE layers, the Top-{\em K} Selector, along with the Gate Network--often referred to as the `router'--selects Experts 1 and 3 based on their scores for the given input. Figure taken from \citep{ai2025mlpfusionefficientfinetuning}.
    }
    \label{fig:moe}
\end{figure}

However, the enormous number of parameters has now become a bottleneck for MoE Transformers~\citep{kong2023swapmoe}, since they require much more GPU memory to load the model even if only part of the parameters are activated each time.
The expert size for Mixtral reaches 176.2M, and the presence of 8 or even more experts in each layer exacerbates the memory demands, bringing a strong need to compress the experts in the MoE structure. 
To give an example, the total model size of Mixtral is 87.0 GB, while the corresponding size of the dense model Mistral is only 13.5 GB.

To leverage the capabilities of MoE LLMs, we revisit several (seemingly unrelated while inherently connected) research avenues below. 
One approach is model fusion \citep{singh2020model,ainsworth2023git}, which involves combining multiple general MLPs. This technique can be adapted to merge experts in MoE models as well. 
More recently, various studies have introduced the concept of expert merging \citep{he2023merging,li2023merge,xue2022student,pmlr-v162-liu22k,stoica2024zipitmergingmodelsdifferent} and expert pruning \citep{lu2024expertsequalefficientexpert}, as a method to reduce the number of experts within each layer of the MoE model. 
Nevertheless, we note the 
direct reduction in the total number of experts 
potentially leads to a substantial loss of the specialized knowledge that individual experts possess (see an illustrative analysis in \Cref{sec:limit}).

To address the aforementioned issues, we introduce \alg, an MoE approximation framework. Our approach capitalizes on approximating the MoE models with fewer parameters by utilizing Wasserstein barycenter techniques~\citep{peyr2020computational}. 
We formulate a distributional representation of experts and extract their common characteristics to obtain the barycenter expert. Subsequently, we propose to employ either unstructured pruning \citep{lee2018snip} or singular value decomposition (SVD)~\citep{NIPS2014_2afe4567} (as a pilot example) to approximate the residual matrices between this barycenter expert and each specific expert.
In summary, the contribution of our work is three-fold:
\begin{itemize}[itemsep=-0.1em, topsep=-0.1em, leftmargin=*]
    \item We introduce Wasserstein barycenter and residual restoration into MoE approximation, aiming to maintain the common and distinctive attributes of each expert with fewer parameters.

    \item We propose \alg, a practical MoE Transformer approximation framework that aims to improve space efficiency in a one-shot and data-agnostic manner, with no extra training required.     
    \item We validate \alg~through extensive experiments on both the encoder-decoder Switch Transformer model, as well as the decoder-only models, Mixtral and DeepSeekMoE. Our results demonstrate that \alg~can reduce the number of parameters in an expert by up to 75\% while incurring only marginal performance loss, verifying its effectiveness and versatility.
\end{itemize}

\section{Related Work}

\textbf{General model compression techniques.}
The focus of deep learning model compression research primarily involves system-level optimization. \emph{Quantization} aims at hardware efficiency by reducing model weight bit-depth from 32-bit floating point (FP32) to 8-bit integers (INT8)~\citep{bhandare2019efficient,dettmers2022llmint8,10.1145/3583780.3615499} or even lower bits ~\citep{li-etal-2022-dq,tao-etal-2022-compression,frantar2023gptq,10.1145/3534678.3539452}. Our focus, however, is on reducing the parameter count of the MoE model, making quantization methods not directly related.

Additionally,
\emph{knowledge distillation} \citep{gou2021knowledge,10.1145/3534678.3539315,10.1145/3340531.3412005} aims to transfer knowledge from pre-trained LLMs to smaller models. However, this approach requires extensive retraining, involving both the original LLM and the compact model.
\emph{Truncated singular value decomposition} (SVD) \citep{denton2014exploiting} has been used to streamline CNNs by reducing redundancy through linear structure exploitation within networks, yet it faces limits in representational capacity, often leading to decreased performance due to overly aggressive dimension reduction. 
\emph{Pruning techniques} \citep{liu2018rethinking,10.1145/3580305.3599284}, evolving with the Lottery Tickets Hypothesis \citep[LTH]{frankle2018lottery}, seek efficient sub-networks within larger models but require extensive retraining to maintain accuracy. 
While some one-shot pruning methods \citep{wang2020picking,sun2024a} do exist, they remain computationally expensive and are not specifically tailored for the structure of MoE, bringing concerns about whether such methods can adequately ensure that the compressed models retain their effectiveness for downstream tasks.

\textbf{Mixture-of-Expert (MoE) transformer compression}.
Ra-ther than applying existing compression techniques individually to the expert MLP, MC-SMoE \citep{li2023merge}, MEO \citep{he2023merging}, and OneS \citep{xue2022student} merge the experts into smaller groups, reducing the count of the experts. 
Expert pruning \citep{lu2024expertsequalefficientexpert,muzio2024seermoesparseexpertefficiency} follows a similar aspect, pruning the less important experts to reduce the size.
This approach faces challenges in deciding the experts to retain, potentially leading to loss of information due to sub-optimal decisions.
\citet{gao-etal-2022-parameter} instead proposed to keep each expert, divide them into several sections, and share the core section among them. This method does not align with our goal
since they aimed to efficiently train a new MoE-like structure from scratch, instead of compressing an existing one.
Alternatively, we note fusion-based methods~\citep{singh2020model,ainsworth2023git}, originally proposed for consolidating distinct models into a single one, 
can be dynamically adapted for consolidating MoE's experts.
These methods utilize the principles of permutation and optimal transport and are implemented layer-wise, which
requires applying the permutations derived from preceding layers to the next one. The characteristic incurs overhead due to the extra time required for permutations.

\section{Preliminaries and Notation}
This section provides the background of MoE, optimal transport, and Wasserstein barycenter.

\subsection{Mixture-of-Experts Modules}

Throughout this paper, %
we consider the classical setting of MoE modules for the ease of analysis, where each expert takes the form of a multilayer perceptron (MLP) in a feed-forward network (FFN) sub-layer of a Transformer. 
It is worth noting that there exist different types of expert network architectures (c.f.\ \Cref{sec:mix_deep}).

Each Mixture-of-Experts (MoE) layer comprises $N$ experts.
The $k$-th expert $E_k$ (a function to transform input vector $\mtx x$ to a new feature) in an FFN sub-layer is denoted as:
\begin{align*}
    E_k(\mtx x) = \mtx W_k^{(2)} \sigma\left(\mtx W_k^{(1)} \mtx x + \mtx b_k^{(1)}\right) + \mtx b_k^{(2)},
\end{align*}
where $\sigma(\cdot)$ is the element-wise activation function. The input $\mtx x\in \mb R^{p}$, and $(\mtx W_k^{(1)}, \mtx b_k^{(1)}) \in \mb R^{p_\text{I} \times (p+1)}$, $(\mtx W_k^{(2)}, \mtx b_k^{(2)}) \in \mb R^{p \times (p_\text{I}+1)}$ are respectively the weight matrices and bias vectors in the linear transforms of the MLP (with input/output dimension $p$ and inner dimension $p_\text{I}$). The output of the MoE layer is given by: $\sum_{k=1}^{N}  [G(\mtx x)]_k \cdot E_k(\mtx x)$.
Here $G(\mtx x) = \operatorname{Softmax}\left(\operatorname{TopK}\left( \mtx W_g \mtx x \right)\right)$ 
returns the normalized sparse router gating vector for all experts, where $\operatorname{TopK}(\mtx g_i)=\mtx g_i$ when $\mtx g_i$ is within the top-{\em k} values of $\mtx g\in \mathbb{R}^N$, otherwise $\operatorname{TopK}(\mtx g_i)=-\infty$;
$\mtx W_g \in \mb R^{N \times p}$ represents the linear transform, turning the input $\mtx x$ into the logit for each expert.
The whole framework of the MoE layer is shown in Fig. \ref{fig:moe}.

The space bottleneck \citep{kong2023swapmoe} comes from the large size of experts (ranging from $8$ to $64$ and even more~\citep{fedus2022switch,dai2024deepseekmoe}) and the tremendous size of the weight matrices in each expert 
(e.g., $176.2$M parameters for each expert in Mixtral~\citep{jiang2024mixtral}).
The sparse design renders the total number of parameters redundant compared to the base dense model. Even though only a part of the parameters is activated each time, the whole model still needs to be loaded in the RAM.
In this paper, we aim to address the redundancy problem while retaining the effectiveness of pre-trained MoE models.

\subsection{Optimal Transport and Wasserstein Barycenter}

Optimal transport (OT) theory has achieved great success in depicting the underlying geometry of distributions~\citep{peyr2020computational}. We consider two distributions $\mu=\sum_{i=1}^n \bm{\alpha}_i \delta_{x_i}$ and $\nu=\sum_{j=1}^m \bm{\beta}_j \delta_{y_j}$ with $\bm{\alpha}_i,\bm{\beta}_j$ as masses  respectively assigned to points $x_i,y_j$, and $\delta_x$ being the Dirac unit mass located on $x$. (In this paper, $\mu, \nu$ will always be the discrete distributions.)
OT reflects a process of transporting the mass from positions $x_i$'s to $y_j$'s (transforming the source distribution $\mu$ to the target distribution $\nu$) with the minimal overall cost, in which the cost of transporting a unit mass from $x_i$ to $y_j$ is given by the cost function $D(x_i,y_j)$. 

A \emph{transport plan} can be specified by a matrix $\mathbf{M}\in\mathbb{R}^{n\times m}$, where $\mathbf{M}_{i,j}$ indicates the mass to be transported from $x_i$ to $y_j$. 
We note that the column and row sums of $\mathbf{M}$ respectively equal to $\bm{\alpha}$ and $\bm{\beta}$, implying all the masses in $\mu$ are transported to the desired points in $\nu$, i.e., $\mathbf{M}\in\Pi(\bm{\alpha},\bm{\beta}) \defeq \{\mathbf{M}\in\mathbb{R}^{n\times m}|\sum_{j=1}^m\mathbf{M}_{i,j}=\bm{\alpha}_{i}, \sum_{i=1}^n\mathbf{M}_{i,j}=\bm{\beta}_{j}\}$. 
OT seeks the optimal plan to transport $\mu$ to $\nu$ w.r.t\ the overall transportation cost, formulated as~\citep{peyr2020computational}:
\begin{align*}
    \text{OT}(\mu,\nu) \defeq \mathop{\argmin}_{\mathbf{M}\in \Pi(\bm{\alpha},\bm{\beta})}\sum_{i,j}\mathbf{M}_{i,j}\mathbf{C}_{i,j}=\underset{\mathbf{M}\in \Pi(\bm{\alpha},\bm{\beta})} \argmin \langle \mathbf{M},\mathbf{C} \rangle,
\end{align*}
where $\mathbf{C}\defeq\left[ D(x_{i},y_{j}) \right]_{ij}\in \mathbb{R}^{{n}\times m}$ is the cost matrix. 

Setting the cost function as $D(x_{i},y_{j}) = \|x_{i}-y_{j}\|^2$,
we can obtain $2$-Wasserstein distance \citep{peyr2020computational} as:
\begin{equation}
W^2_2(\mu,\nu)\defeq
\underset{\mathbf{M}\in \Pi(\bm{\alpha},\bm{\beta})}{\min}\left \langle \mathbf{M},\mtx C\right \rangle. \notag
\end{equation}
In this paper, we specifically focus on the free-support Wasserstein barycenter problem induced by 2-Wasserstein distance.
Given a set of distributions $\mu_1$, \dots, $\mu_N$, the Wasserstein barycenter $\bar{\mu}$ is the ``average'' distribution in terms of the Wasserstein distance. 
To regulate the form of $\bar{\mu}$ in numerical computation, we specify $\bar{\mu}$ as a uniform distribution on $n$ points,
and optimize it through:
\begin{equation}
\label{eq:wb_p}
    \bar{\mu} = \underset{\{x_i\}_{i=1}^{n}}{\argmin} \frac{1}{N}\sum_{k=1}^{N}  W_2^2 \paren{\mu_k, \sum_{i=1}^{n}\frac{1}{{n}}\delta_{x_i}}.
\end{equation}
We comment \citet{cuturi2014fast} have provided efficient numerical algorithms for the free-support Wasserstein barycenter problem above, which will be heavily utilized in our implementations. 
In a nutshell, we conclude Wasserstein barycenter captures the underlying advection in the distribution space, offering a powerful tool for aggregating distributions in a geometrically meaningful way.

\section{Proposed Methodology}

\begin{figure}[tbp]
\centering
  \includegraphics[width=1\columnwidth]{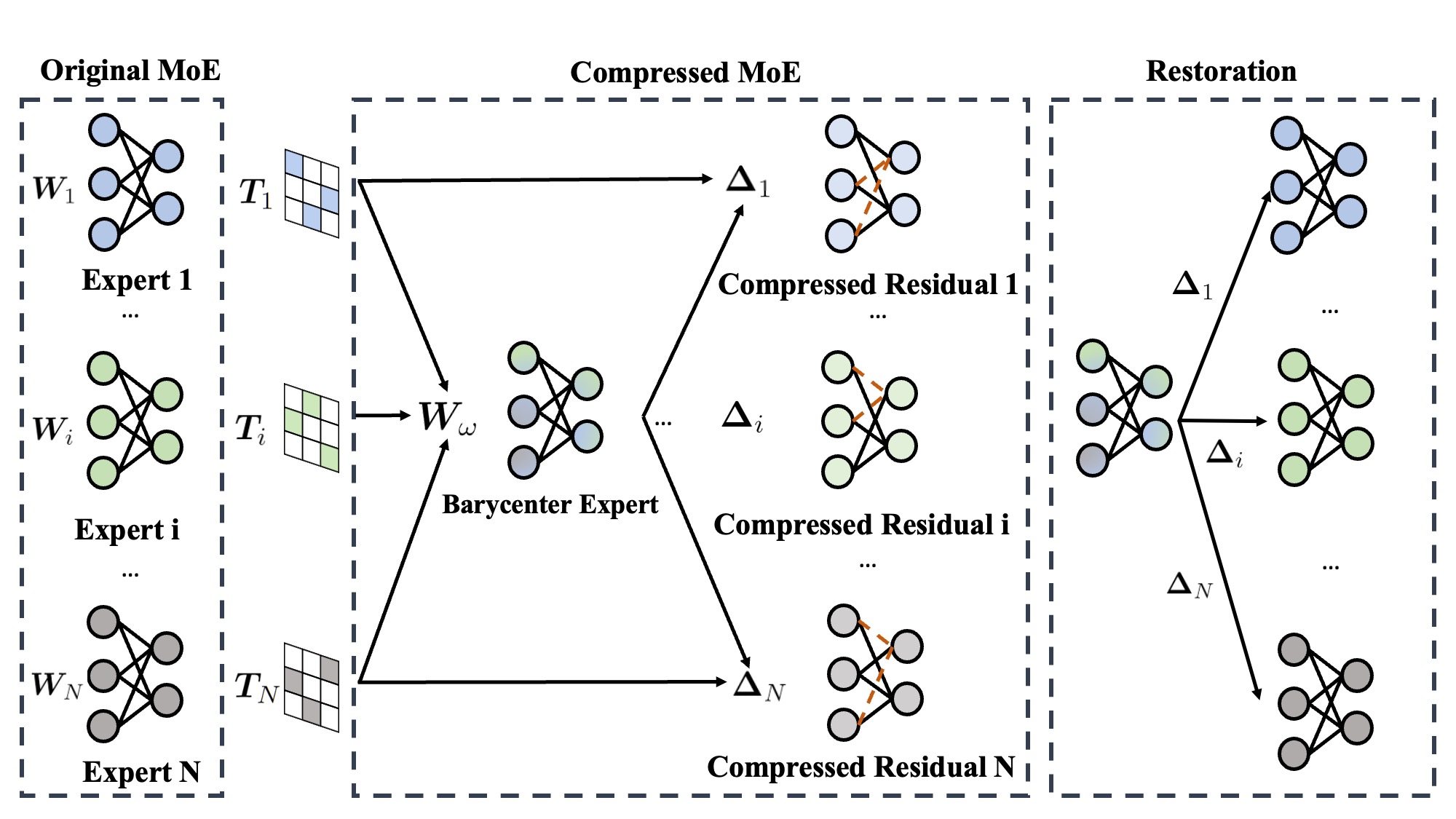}
  \caption{The overall framework of \alg. We introduce permutation matrices $\mtx T$ to obtain the barycenter expert $\mtx {W_{\omega}}$ from a distributional view. Instead of compressing the original experts directly, we opt to compress the residual matrices ($\mtx \Delta$, illustrated with lighter colors) between each expert and the barycenter expert, with the capability to dynamically and efficiently restore the original matrices during inference. We illustrate the concept using unstructured pruning as an example, with dashed orange lines indicating the pruned connections within the network.}%
\label{fig:resmoe}
\end{figure}

In this section, we first analyze the limitations of existing fusion strategies, and then give a detailed introduction to \alg, along with its visualization provided in Figure \ref{fig:resmoe} and pseudocode in \Cref{app:pseudocode}. For a comprehensive visual comparison between \alg~and previous baseline methods, please refer to
Figure \ref{fig:compare}.

\subsection{Limitations of Existing Fusion Strategies}
\label{sec:limit}
In advance of our proposal in \Cref{sec:strategy}, we first review the limitations of existing fusion/merge strategies, 
which mainly serves as the motivation for developing the ResMoE framework.

Alignment-based model fusion~\citep{singh2020model,ainsworth2023git} is proposed to fuse multiple models, that can be adapted in the MoE structure to fuse MLPs (which may have more than two layers) and relates to OT.
In OT Fusion~\citep{singh2020model}, for a two-layer MLP, their algorithm starts from aligning the first layer of each expert and then pre-aligns the second layer with the permutation matrix obtained from the first layer, repeating the procedure to align the second layer further.
Similarly, in Git Re-Basin \cite{ainsworth2023git}, they propose to align the weights through a greedy loop for each layer, which demonstrates zero-barrier linear mode connectivity~\citep{10.5555/3524938.3525243} between independently trained models on the same dataset. We note the alignment-based model fusion technique implies a layer-by-layer strategy to formulate MoE layers as distributions and to merge experts. 
Moreover, expert merging~\citep{he2023merging,li2023merge,xue2022student}
has been recently introduced to reduce the number of experts in MoE modules. This approach merges experts using task-specific information, such as router gating score distribution or router activation frequency.

Both the model fusion and expert merging reduce the number of experts to compress the model, while we remark that the direct reduction in the number of experts may lead to a huge deviation from the original module output, especially in the zero-shot setting.

To better understand this, we first revisit the MoE layer from an all-experts matrix perspective. 
We define the router matrix $\mtx R$ as:

\begin{small}
\begin{align*}
\mtx R \defeq \operatorname{diag}(G(\mtx x)) \otimes \mtx I_{p_{\text {I }}}=\begin{bmatrix}
	  [G(\mtx x)]_1\mathbf{I}&   & &  \\
	  & ... &  & \\
	  &   & [G(\mtx x)]_N \mtx I
	 \end{bmatrix}_{N\cdot p_{\text{I}} \times N\cdot p_{\text{I}}},
\end{align*}
\end{small}%
where $\otimes$ is the Kronecker product, $\mtx I$ is the identity matrix, and $G(\mtx x)$ is a sparse score vector. The weight matrices in the MoE layer are overall denoted as:
\begin{align*}
&\mtx W^{(1)}=\left(\begin{array}{lll}
    \mtx W_1^{(1)}  \cdots  \mtx W_N^{(1)}
    \end{array}\right)_{ N \cdot p_{\text {I}}\times p}^\T, \\
    &\mtx W^{(2)}=\left(\begin{array}{lll}
    \mtx W_1^{(2)}  \cdots  \mtx W_N^{(2)}
    \end{array}\right)_{p \times N \cdot p_{\text {I }}}.
\end{align*}
The output of the MoE layer can accordingly be expressed as (omitting the bias term for simplicity):
\begin{align}
    \mtx y = \mtx W^{(2)} \mtx R \sigma\left(\mtx W^{(1)} \mtx x\right),
\label{eqn:moe_y}
\end{align}
in which $\mtx R$ encapsulates crucial expert knowledge and exhibits high sparsity, as typically only a selected number of experts are activated within each MoE layer.

To analyze the difficulty of compressing the space-occupying $N\cdot p_\text{I} \times p$ matrices $\mtx W^{(1)}, (\mtx W^{(2)})^\T$, we turn to a theoretical framework \emph{oblivious subspace embedding}~\citep[OSE]{cohen2016nearly}
which can preserve any $p$-dimensional subspace of $\mb R^{N\cdot p_\text{I}}$ (here, we focus on the subspaces spanned by $\mtx W^{(1)}$ and $(\mtx W^{(2)})^\T$) through a $d \times N\cdot p_\text{I}$
random projection matrix $\mtx \Pi$, with $d$ being the projection dimension.
The projection $\mtx \Pi \mtx W^{(1)}$ and $\mtx \Pi (\mtx W^{(2)})^\T$ can be considered as a sketch of the expert merging strategy.

We then recognize the limitation of expert merging via the lens of OSE.
As per \citet{cohen2016nearly}, $d$ should be at least $\m O(p \log p / \varepsilon^2)$,
where $\varepsilon$ is the error tolerance level;
in compressing MoE ($N$ will NOT go to infinity), however, the scale of $p \log p / \varepsilon^2$ will be even larger than $N\cdot p_\text{I}$ by simply setting $\varepsilon=0.05$.
The space gain from expert merging is thus usually marginal in this very practical case. In this regard, reducing the number of experts per MLP layer might not be that practical. Instead, we propose each expert should be kept during compression.

To alleviate the deficiency of OSE above, MC-SMoE \citep{li2023merge}
leverages the information from training data to reduce the scale of $d$, rendering the merging no longer data-``oblivious''.
However, due to the requirement to first do fine-tuning and the restriction that during inference the test data have to be i.i.d.\ as the training data,
this therapy will be less valid in the zero-shot setting we mainly consider.
We also empirically verify the above observations on the expert merging strategies in \Cref{sec:nlg_exp}.

\subsection{An Extraction Strategy Specific to the MoE Structure}
\label{sec:strategy}
Following our speculation that each expert in the MLP layer should be kept, 
we propose \alg, a framework to compress the representation of those experts through a Wasserstein barycenter expert and the residual matrices between each expert and the barycenter expert.
Specifically, we extract an expert $E_\omega$ with the common pattern from all the experts, 
and then model the difference between $E_\omega$ and $E_k$ by fewer parameters (we will introduce the difference modeling in \Cref{sec:diff}).
We first revisit a viewpoint that an MLP can be taken as the ensemble of multiple bottleneck-1 sub-MLPs~\citep{chen2023ntk, wang2022exploring, yuan2022distributed}.
We rewrite the MLP output as follows:
\begin{align}
E_k(\mtx x) = \sum_{i=1}^{p_\text{I}} \left[\mtx W_{k, \cdot, i}^{(2)} \cdot \sigma\paren{
\dotp{\mtx W_{k, i, \cdot}^{(1)}}{\mtx x} + \mtx b_{k, i}^{(1)}}  \right] + \mtx b_k^{(2)},
\label{eqn:mlp_sum}
\end{align}
where by convention we represent the $i$-th row (resp.\ column) in the weight matrix $\mtx W_k^{(1)}$ (resp.\ $\mtx W_k^{(2)}$) as $\mtx W_{k, i, \cdot}^{(1)}$ (resp.\ $\mtx W_{k, \cdot, i}^{(2)}$), and $\sigma(\cdot)$ is the activation function.
The summation implies that an MLP is the ensemble of a few bottleneck-1 sub-MLPs (the sum on the right-hand-side above), which allows a distributional perspective of MLP since the order of the sum does not matter.

Note that various expert network architectures can all be expressed using the structure of multiple bottleneck-1 sub-MLPs. The FFN in both Mixtral and DeepSeekMoE models uses a gated network following Llama \citep{touvron2023llama}, 
whose detailed form can be found in \Cref{sec:mix_deep}.

Since $\mtx b_i^{(2)}$ is not involved in the summation in \Cref{eqn:mlp_sum},
we accordingly quantify the extraction as, after proper permutation, minimizing the squared Frobenius norm of differences between the original weight matrices in each expert and the weight matrices in the barycenter expert:

\begin{align}
\label{eqn:min_frob}
    \min_{\substack{\mtx W_\omega^{(1)}, \mtx b_\omega^{(1)}, \mtx W_\omega^{(2)} \\
    \mtx T_k \in \m P, k \in [N]}}
    &\frac{1}{N}\sum_{k=1}^N \bigg[\left\|\mtx T_k \bigg[\mtx W_k^{(1)}, \mtx b_k^{(1)}\bigg] - \bigg[\mtx W_\omega^{(1)}, \mtx b_\omega^{(1)}\bigg] \right\|_F^2  \\
    &\qquad +\left\|\mtx W_k^{(2)} \mtx T_k^\T - \mtx W_\omega^{(2)} \right\|_F^2\bigg],
    \notag
\end{align}
where $\m P$ is the class of $p_\text{I}$-by-$p_\text{I}$ permutation matrices and $\mtx W_\omega^{(1)} \in \mb R^{p_\text{I} \times p}, \mtx b_\omega^{(1)} \in \mb R^{p_\text{I}}, \mtx W_\omega^{(2)} \in \mb R^{p \times p_\text{I}}$ are the weight matrices in the barycenter expert $E_\omega$.
The introduction of the permutation matrices $\mtx T_k$'s aligns with the distributional perspective of MLPs, that an MLP $E_k$ is equivariant to the row permutation of its design matrix $\mtx W_k= \brkt{\mtx W_k^{(1)}, \mtx b_k^{(1)}, (\mtx W_k^{(2)})^\T} \in \mb R^{p_\text{I} \times (2p + 1)}$ as the sum's order in \Cref{eqn:mlp_sum} is inconsequential.
It is worth noting that simultaneously permuting $\mtx W_k^{(1)}$ and $\mtx W_k^{(2)}$ does not affect the expert's output since the permutation matrix is orthogonal.

To solve problem~\eqref{eqn:min_frob},
we propose to address the distribution of sub-MLPs within each expert, rather than layer-by-layer. The summation in
\Cref{eqn:mlp_sum} clearly shows the correspondence between the $i$-th row $\mtx W_{k, i, \cdot}^{(1)}$ of $\mtx W_k^{(1)}$ and the $i$-th column $\mtx W_{k, \cdot, i}^{(2)}$ of $\mtx W_k^{(2)}$; to obtain the ``embedding'' of the sub-MLPs for the distributional formulation, 
we consider the original MLP $E_k$ as a design matrix $\mtx W_k$.
We then run the algorithm for free-support Wasserstein barycenter~\cite{cuturi2014fast} to obtain the weight matrix $\mtx W_\omega = \left[\mtx W_\omega^{(1)}, \mtx b_\omega^{(1)}, (\mtx W_\omega^{(2)})^\T\right]$ for the barycenter expert,
with $\mtx W_\omega$ being exactly the solution to the minimization problem \eqref{eqn:min_frob}.

To present the result, we respectively define $\mu_k$'s as the uniform distributions defined on the rows of the given $\mtx W_k \in \mb R^{p_\text{I} \times (2p+1)}$, for all $k = 1, 2, \cdots, N$, i.e., $\mu_k=\sum_{i=1}^{p_\text{I}} {1}/{p_\text{I}} \cdot \delta_{\mtx{W}_{k,i,\cdot}}$.
Similarly $\mu_\omega$ is uniformly distributed over the rows of $\mtx W_\omega \in \mb R^{p_\text{I} \times (2p+1)}$, 
and the notation $\mu_\omega$ is interchangeable with $\mu_\omega(\mtx W_\omega)$, which highlights the dependence on $\mtx W_\omega$.
We further denote the optimal transport matrix (w.r.t.\ $W_2$ distance) from $\mu_k$ to $\mu_\omega$ as $\text{OT}(\mu_k, \mu_\omega)$ as the solution of 
\Cref{eq:wb_p}.
We can then give the following proposition (the proof is deferred to \Cref{sec:proof_equivalence}).

\begin{restatable}[]{proposition}{equivalence}
Consider the solution $\mtx W_\omega$ to the following free-support WB problem
\begin{align}
\label{eq:wb}
    \min_{\mtx W_\omega} \frac{1}{N} \sum_{k=1}^N W_2^2(\mu_k, \mu_\omega(\mtx W_\omega)).
\end{align}
Then $\mtx W_\omega$, along with $\mtx T_k = p_{\operatorname{I}} \cdot \operatorname{OT}\left(\mu_k, \mu_\omega(\mtx W_\omega)\right)$, is the solution to the optimization problem \eqref{eqn:min_frob}.
\label{prop:equivalence}
\end{restatable}

\begin{figure*}[t!]
\centering
  \includegraphics[width=1.7\columnwidth]{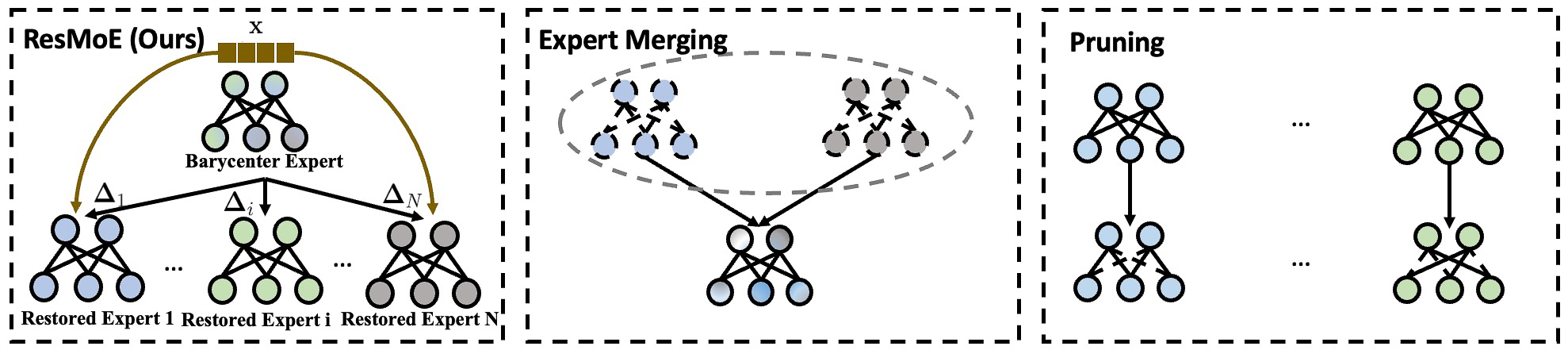}
  \caption{Comparisons between \alg~and baselines. Dash lines denote the connections or neurons are deleted. Expert Merging reduces the number of experts by consolidating several into one, while pruning is applied directly to the experts. In contrast, \alg~compresses the residual and barycenter experts, with the input x directed to the restored experts.}
\label{fig:compare}
\end{figure*}
\noindent \textbf{Remark}. 
We note that all the experts $E_k$ and the barycenter expert $E_\omega$ share the same size, i.e., $\mathbf{W}_k,\mathbf{W}_\omega\in\mathbb{R}^{p_\text{I}\times (2p+1)}$. Therefore, the supports for distributions $\mu_k,\mu_\omega$ are of the same size $p_\text{I}$. In discrete optimal transport,
there is a special property that for two discrete uniform distributions with support of the same size, 
the optimal transport matrix between the two distributions will be re-scaled as a permutation matrix \citep{peyr2020computational}.
The conclusion simplifies the computation since the permutation matrix is orthogonal ($\mtx T_k \left(\mtx T_k\right)^\T = \mtx I$).
The output of the extracted expert $E_\omega(\mtx x) = \mtx W_\omega^{(2)} \sigma\big(\mtx W_\omega^{(1)} \mtx x + \mtx b_\omega^{(1)} \big)$ (adding $\mtx b_\omega^{(2)}$) is automatically aligned with any expert $E_k(\mtx x), \forall k \in [N]$ without additional transformations.

\subsection{Residual Approximation and Expert Restoration}

\label{sec:diff}

As the last step, we need to recover the selected expert from the barycenter expert. We choose two representative methods: unstructured pruning and SVD to remove the redundant parameters. 
(For unstructured pruning, we follow \citet{han2015learning} to zero out the parameters with small magnitude, in order to minimize the loss in problem~\eqref{eqn:min_frob}.)

We will store a compressed matrix $\mtx \Delta_k$ to approximate $\mtx T_k \mtx W_k - \mtx W_\omega$ and then use $\mtx \Delta_k + \mtx W_\omega$ to recover $\mtx T_k \mtx W_k$. We provide more implementation tricks for them in \Cref{sec:memory} and the pseudocode of the algorithm in \Cref{app:pseudocode}. We remark that unstructured pruning produces comparable experimental results while SVD leads to more profound memory reduction.

In \cref{fig:compare}, we visually compare \alg~with previous baselines. Expert merging techniques consolidate multiple experts into a single entity, while pruning strategies involve the direct removal of connections within individual experts.
\alg~first obtains the common barycenter expert and subsequently compresses the residuals between each expert and the barycenter expert, which can be effectively and efficiently used for restoring in the inference stage.

\section{Experimental Results}
\label{sec:exp}

\begin{table}[t!]
\caption{Approximation error of Switch Transformer and Mixtral. The experts are frozen during the fine-tuning stage hence most of the deterministic methods give zero standard deviation. All the numbers are normalized by a factor $p_{\text{I}}$ for better reference. UP stands for Unstructured Pruning, and SP stands for Structured Pruning.}
\label{tab:app_err_per}
\centering
\begin{tabular}{lcc}
\toprule
 & Switch Transformer & Mixtral\\
\midrule
UP & 34.27$\pm$0.00  &10.26$\pm$0.00\\ 
Wanda & 22.93$\pm$0.03  & 13.47$\pm$0.00\\  
SP & 87.00$\pm$0.00 & 27.09$\pm$0.00\\
SVD           & 56.44$\pm$0.00 & 21.70$\pm$0.00 \\
M-SMoE         & 278.76$\pm$0.00 &16.73$\pm$0.00\\
MEO             &  63.25$\pm$0.00& 15.81$\pm$0.00 \\
MLP Fusion          &  83.45$\pm$0.02&27.28$\pm$0.01 \\
\midrule
\alg~(UP)    & \textbf{22.05$\pm$0.02} & \textbf{6.60$\pm$0.01} \\
\alg~(SVD)      & 48.91$\pm$0.01 & 14.63$\pm$0.04\\
\bottomrule
\end{tabular}
\end{table}

This section starts with our experimental setup, followed a preliminary evaluation of \alg's approximation error against various baselines in Table \ref{tab:app_err_per} and its
performance in Tables \ref{tab:switch} and \ref{tab:mixtral}, concluding with an ablation study. All models and methods are implemented in PyTorch. Switch Transformer is fine-tuned on a Tesla V100 32GB GPU, while Mixtral is tested on four such GPUs. Detailed experimental information is available in Appendix \ref{sec:finetune}.

\subsection{Experiment Setup}

\textbf{Model backbones.}
Our evaluation encompasses two primary architectures: the GPT-style causal decoder-only model and the T5-style encoder-decoder model. Specifically, we utilize Mixtral \citep{jiang2024mixtral} for the decoder-only model, featuring 8 experts per layer across 32 layers. 
For the encoder-decoder model, we employ the Switch Transformer \citep{fedus2022switch} with a similar expert-layer configuration (switch-base-8) but with 12 encoder layers followed by 12 decoder layers.
We fix the router and the experts during the supervised fine-tuning stage, based on the observation that preserving the original LLM's universal world information can enhance their performance~\citep{he2023preserving,mukhoti2023finetuning,dou2023loramoe}. This observation
is empirically supported by our findings, which demonstrate improved model performance. 
We also provide the efficiency analysis in Appendix \ref{app:efficiency}.

\begin{table*}[t]
\centering
\caption{
Evaluation results of Switch Transformer on four GLUE NLU tasks (measured in accuracy).
UP stands for Unstructured Pruning, and SP stands for Structured Pruning. We use ``concat'' and ``sep'' to denote the concatenated and separate processing of the expert weights. \textbf{Bold} indicates the best score for each metric, while \textbf{underlined} values represent the second-best.}
\label{tab:switch}
\begin{tabular}{lcccc}
\toprule
&SST-2 & MRPC & CoLA & MNLI\\
\midrule
Switch Transformer &93.92$\pm$0.18  & 89.54$\pm$0.86  & 82.29$\pm$0.28 & 87.82$\pm$0.15 \\
\midrule
UP (concat) &93.12$\pm$0.23  &87.75$\pm$1.12 &81.40$\pm$0.48 &85.32$\pm$0.66 \\
UP (sep) & 90.21$\pm$0.44 &78.92$\pm$3.96 &79.33$\pm$0.80 & 76.06$\pm$5.47\\
Wanda & 92.39$\pm$0.18 &86.74$\pm$0.93 &80.59$\pm$1.01 & 84.20$\pm$0.19\\
SP (concat)&90.67$\pm$0.36  &\underline{88.72$\pm$0.85}  &79.39$\pm$0.42  &85.19$\pm$0.22\\
SP (sep)&83.60$\pm$1.15  &80.88$\pm$2.12  &76.89$\pm$0.93  &81.87$\pm$1.19\\
SVD (concat)&92.47$\pm$0.04 & 87.58$\pm$1.02 &75.90$\pm$0.03 & 85.86$\pm$0.07\\
SVD (sep)&92.59$\pm$0.25 &87.25$\pm$0.93  &\underline{81.62$\pm$0.32} &86.04$\pm$0.05 \\
M-SMoE &\underline{93.31$\pm$0.53} &87.42$\pm$1.06 &80.06$\pm$0.68 &85.72$\pm$0.27 \\
Git Re-Basin & 84.94$\pm$0.86&85.70$\pm$0.46 &61.90$\pm$1.24 & 83.76$\pm$0.84\\
MEO &92.73$\pm$0.39 &86.77$\pm$0.93 &79.99$\pm$0.83 &85.29$\pm$0.37 \\
MLP Fusion &91.86$\pm$0.30 &88.40$\pm$0.59 &79.64$\pm$0.03 &85.72$\pm$0.19 \\
\midrule
\alg~(UP)&\textbf{93.58$\pm$0.07} & \textbf{89.21$\pm$0.49} &\textbf{ 82.13$\pm$0.07} & \textbf{86.13$\pm$0.09}\\
\alg~(SVD) &92.85$\pm$0.05  & 88.18$\pm$0.48  &76.88$\pm$0.08  &\underline{86.08$\pm$0.03} \\ 
\bottomrule
\end{tabular}
\end{table*}

\textbf{Compared methods.}
We compare our method with different types of baselines. For pruning, we employ both single-shot unstructured pruning~\citep{NIPS2015_ae0eb3ee,lee2018snip,NEURIPS2020_46a4378f} and structured pruning~\citep{li2023losparse}. We also employ Wanda \citep{sun2024a} for a more enhanced unstructured pruning method.
We employ truncated SVD following~\citet{NIPS2014_2afe4567}. For merging, we employ M-SMoE  (the better-performing uncompressed version of MC-SMoE) \citep{li2023merge} and MEO \citep{he2023merging}. As the experimental results drop is profound in expert pruning \citep{lu2024expertsequalefficientexpert} (50\%), we employ it here only to Mixtral since our compression rate is more extreme (25\%).
We employ Git Re-Basin \citep{ainsworth2023git} for model fusion.
Note that Git Re-Basin is not initially designed for MoE models, and we dynamically apply it as a fusion (merging) method according to its applicability to merge multiple models.
We also compare our method with MLP Fusion \citep{ai2025mlpfusionefficientfinetuning}, which aims to reduce the intermediate dimension of one expert MLP unit. 
As our compression rate is set to 25\%, we perform different setups to all the methods to make sure they match this setting, detailed in Appendix \ref{sec:compress_ratio}.

\subsection{Preliminary Evaluation of Approximation Error}
\label{subsection:approx}

We calculate the approximation error of each method on the top 8 layers of Switch Transformer and the top 24 layers of Mixtral as a sanity check. The approximation error is defined as the Frobenius norm difference between the original and compressed weight matrices in the experts.
Take Switch Transformer for example, since there is no bias in this model, the approximation error $\epsilon$ for one layer is defined as:
\begin{align*}
    \epsilon = \frac{1}{N} \sum_{k=1}^N \bigg[\left\|\mtx T_k \mtx W_k^{(1)} - \hat{\mtx W}_k^{(1)} \right\|_F^2 +\left\|\mtx W_k^{(2)} \mtx T_k^T - \hat{\mtx W}_k^{(2)} \right\|_F^2\bigg], 
\end{align*}
where $\hat{\mtx{W}}_k$ denotes the matrix post-application of each method. For \alg, $\hat{\mtx W}=\mtx W_\omega+\mtx \Delta_k$, with $\mtx{\Delta}_k$ being the compressed residual matrix derived from each layer, and $\mtx{W}_\omega$ as the Wasserstein barycenter matrix. For merge methods, $\hat{\mtx W}=\mtx W_\omega$, where $\mtx W_\omega$ is the merged center of each group. For methods not involving permutation operations, we set $\mtx{T}_k=\mtx{I}$. Specifically for MLP fusion which reduces the MLP's weight matrix size, it still allows approximation error computation, as detailed in \Cref{sec:approx}.
Notably, as we freeze the experts during fine-tuning, most methods show zero standard deviation. It is worth mentioning that given Wanda is not data-agnostic, it has a standard deviation for different tasks Switch Transformer is fine-tuned on. As for Mixtral, we follow the zero-shot setting of \citet{sun2024a} to use the C4 dataset \citep{raffel2023exploringlimitstransferlearning} to perform the algorithm, hence leading to the zero standard deviation of it on Mixtral.

Table \ref{tab:app_err_per} shows that \alg~achieves the lowest Approximation error among all the methods. We use the acronyms UP to represent Unstructured Pruning and SP for Structured Pruning. The results prove that ResMoE manages to retain not only the output of the original model but also the integrity of the internal matrices. Thus, this preliminary experiment successfully validates Proposition~\ref{prop:equivalence}.

\subsection{Natural Language Understanding}

\textbf{Experiment setup.} Switch Transformer is fine-tuned then compressed during the inference stage on four natural language understanding (NLU) GLUE tasks, SST-2 \citep{socher2013recursive}, MRPC \citep{dolan-brockett-2005-automatically}, CoLA \citep{warstadt-etal-2019-neural} and MNLI \citep{williams-etal-2018-broad}. All the results are reported with accuracy. Here, all the experiments are conducted using different seeds for three rounds. As not all the layers of Switch Transformer are sparse MoE layers, we perform all the methods at the top 4 encoder's MoE layers and the top 4 decoder's MoE layers. 

\begin{table*}[t]
\centering
\caption{Zero-shot results of Mixtral. Most of the methods are deterministic based on the model's weights, resulting in a 0 standard deviation. \textbf{Bold} indicates the best score for each metric, while \textbf{underlined} values represent the second-best. On the WikiText dataset, where perplexity serves as the evaluation metric. The down-arrow notation (\textdownarrow) indicates that a lower metric represents better performance.
}
\label{tab:mixtral}
\begin{tabular}{lcccc}
\toprule
 & WikiText (PPL) \textdownarrow & LAMBADA (ACC) & PIQA (ACC) & WinoGrande (ACC)  \\
\midrule
Mixtral        & 3.87$\pm$0.00  & 74.05$\pm$0.00 & 82.37$\pm$0.00 & 77.11$\pm$0.00  \\
\midrule
UP        & 13.03$\pm$0.00 & 36.10$\pm$0.00 & 72.09$\pm$0.00 & 68.59$\pm$0.00  \\
Wanda & 34.57$\pm$0.00 & 18.73$\pm$0.00 & 63.82$\pm$0.00 & 59.75$\pm$0.00 \\
SP        & 13851.63$\pm$0.00 & 0.00$\pm$0.00 & 53.05$\pm$0.00 & 47.91$\pm$0.00  \\
SVD            & 267.94$\pm$0.00& 16.09$\pm$0.00 & 59.47$\pm$0.00 & 56.99$\pm$0.00  \\
M-SMoE         & 10.45$\pm$0.00 & 58.57$\pm$0.00 & 73.56$\pm$0.00 & 69.61$\pm$0.00  \\
Git Re-Basin   & 9.96$\pm$0.00  & 59.09$\pm$0.00 & 74.70$\pm$0.00 & 69.22$\pm$0.00  \\
MEO            & 8.32$\pm$0.00 & 62.93$\pm$0.00 & 75.84$\pm$0.00 & 70.48$\pm$0.00  \\
Expert Pruning & 8.14±0.00 & 59.07±0.00	& 76.82±0.00 & 70.88±0.00 \\
MLP Fusion         & 80.06$\pm$5.55 & 5.12$\pm$0.49  & 66.67$\pm$0.25 & 56.80$\pm$0.97  \\
\midrule
\alg~(UP)   & \textbf{5.38$\pm$0.04} & \textbf{69.44$\pm$0.16} & \textbf{80.81$\pm$0.19} & \textbf{74.45$\pm$0.23}  \\
\alg~(SVD)   & \underline{7.26$\pm$0.05} & \underline{64.72$\pm$0.15} & \underline{78.02$\pm$0.14} & \underline{73.16$\pm$0.09}  \\
\bottomrule
\end{tabular}
\end{table*}

\begin{table*}[t]
\caption{The comparison of accuracy
between vanilla pruning, average expert, Git Re-Basin expert, vanilla SVD, and our method. Here UP means Unstructured Pruning, WB stands for Wasserstein barycenter. \textbf{Bold} results are better scores
under each metric.}
\label{tab:ablation_rslt}
\centering
\begin{tabular}{lcccccc}
\toprule
& \multicolumn{3}{c}{Switch Transformer} & \multicolumn{3}{c}{Mixtral} \\ 
  & SST-2 & MRPC & MNLI & LAMBADA & PIQA & WinoGrande\\ 
  \midrule
UP &93.12$\pm$0.18  & 87.75$\pm$1.12 & 85.32$\pm$0.66  & 36.10$\pm$0.00 & 72.09$\pm$0.00 & 68.59$\pm$0.00\\
Avg + UP  &92.81$\pm$0.42  & 89.13$\pm$0.96  & 86.00$\pm$0.18 & 67.38$\pm$0.00 & 78.89$\pm$0.00 & 73.95$\pm$0.00  \\
Git + UP  & 92.62$\pm$0.17  & 88.89$\pm$0.56   & \textbf{86.23$\pm$0.13} &  46.11$\pm$0.00 & 70.95$\pm$0.00  &  67.72$\pm$0.00 \\
WB + UP &\textbf{93.58$\pm$0.07} & \textbf{89.21$\pm$0.49} &86.13$\pm$0.09  &  \textbf{69.44$\pm$0.16} & \textbf{80.81$\pm$0.19} & \textbf{74.45$\pm$0.23}\\
\midrule
SVD &92.47$\pm$0.04 & 87.58$\pm$1.02  &85.86$\pm$0.07   & 16.09$\pm$0.00 & 59.47$\pm$0.00 & 56.99$\pm$0.00\\
WB + SVD &\textbf{92.85$\pm$0.05}  & \textbf{88.18$\pm$0.48}  &\textbf{86.08$\pm$0.03}     & \textbf{64.72$\pm$0.15} & \textbf{78.02$\pm$0.14} & \textbf{73.16$\pm$0.09}  \\
\bottomrule
\end{tabular}
\end{table*}

\textbf{Results.} Table \ref{tab:switch} provides the results of Switch Transformer. \alg~(UP) consistently surpasses all baseline methods, while \alg~(SVD) manages to surpass most of the baseline methods, underscoring its efficiency.
Unstructured pruning effectively preserves the original performance, whereas structured pruning, applied neuron-wise, exhibits a more pronounced drop. This observation aligns well with our choice of unstructured pruning over structured pruning.
We also observe that Wanda performs even worse than vanilla unstructured pruning. This may be due to the fact that \citet{sun2024a} set the compression ratio to 50\%, while our setting retains only 25\% of the parameters, leading to a more significant performance drop.
We note a significant difference in performance when applying pruning and SVD to experts, depending on whether the weights were concatenated or separate. 
A possible explanation is that 
pruning dynamically zeroes out less important weights, retaining crucial ones when expert weights are concatenated, indicating the benefit of preserving expert-level relationships for model performance.
In addition, the suboptimal results from Git Re-Basin further support our proposition that previous layer-wise fusion methods are limited in their effectiveness. These methods may fail to adequately capture the complexities of layer interactions, leading to less optimal outcomes when compared to more holistic approaches.
Although most methods manage to preserve the model's performance well in these NLU tasks, they result in a dramatic drop in subsequent zero-shot natural language generation (NLG) tasks.

\subsection{Zero-shot Natural Language Generation}
\label{sec:nlg_exp}

\textbf{Experiment setup.} Mixtral is tested on WikiText (Language Modelling)~\citep{merity2016pointer}, LAMBADA (Language Modelling)~\citep{paperno-etal-2016-lambada}, PIQA (Question Answering)
\citep{bisk2020piqa} 
and WinoGrande (Common Sense Reasoning)~\citep{sakaguchi2021winogrande}. The result of WikiText is given by perplexity, while accuracy metrics are used for the others.
As Mixtral's results are tested with zero-shot and fixed weights, this ensures deterministic outcomes for most of the evaluated methods, leading to a standard deviation of 0 for them. However, the Fusion and OT methods, which seek approximate optimization solutions starting from different initial conditions, exhibit variability and therefore have a non-zero standard deviation. Specifically, for Wanda, we follow the zero-shot setting of \citet{sun2024a} to perform the algorithm on the C4 dataset \citep{raffel2023exploringlimitstransferlearning}.
All the methods are performed on the top 24 layers, and reduce the parameter counts of the experts to 25\%.

\textbf{Results.}
Table \ref{tab:mixtral} presents the results for Mixtral, where both \alg~(UP) and \alg~(SVD) consistently outperform all baseline methods, demonstrating their effectiveness in both NLU and NLG tasks.
Notably, structured pruning results in a substantial performance loss for Mixtral, likely due to its larger hidden dimension (4,096), where neuron-wise weight pruning could lead to significant information loss, a situation reminiscent of the MLP Fusion case. It is important to note that Mixtral's experts are initialized through a copy-and-paste method, as opposed to the random Gaussian initialization in Switch Transformer, leading to more uniform weight distributions in Mixtral. 
This uniformity might contribute to the enhanced performance observed with merge methods in Mixtral. However, the superior performance of \alg~over merge methods further supports our hypothesis about the latter's reduced effectiveness in more generalized scenarios.

\subsection{Ablation Studies}
\label{sec:ablate}

\begin{figure}[h]
\begin{center}
\centerline{\includegraphics[width=1.0\columnwidth]{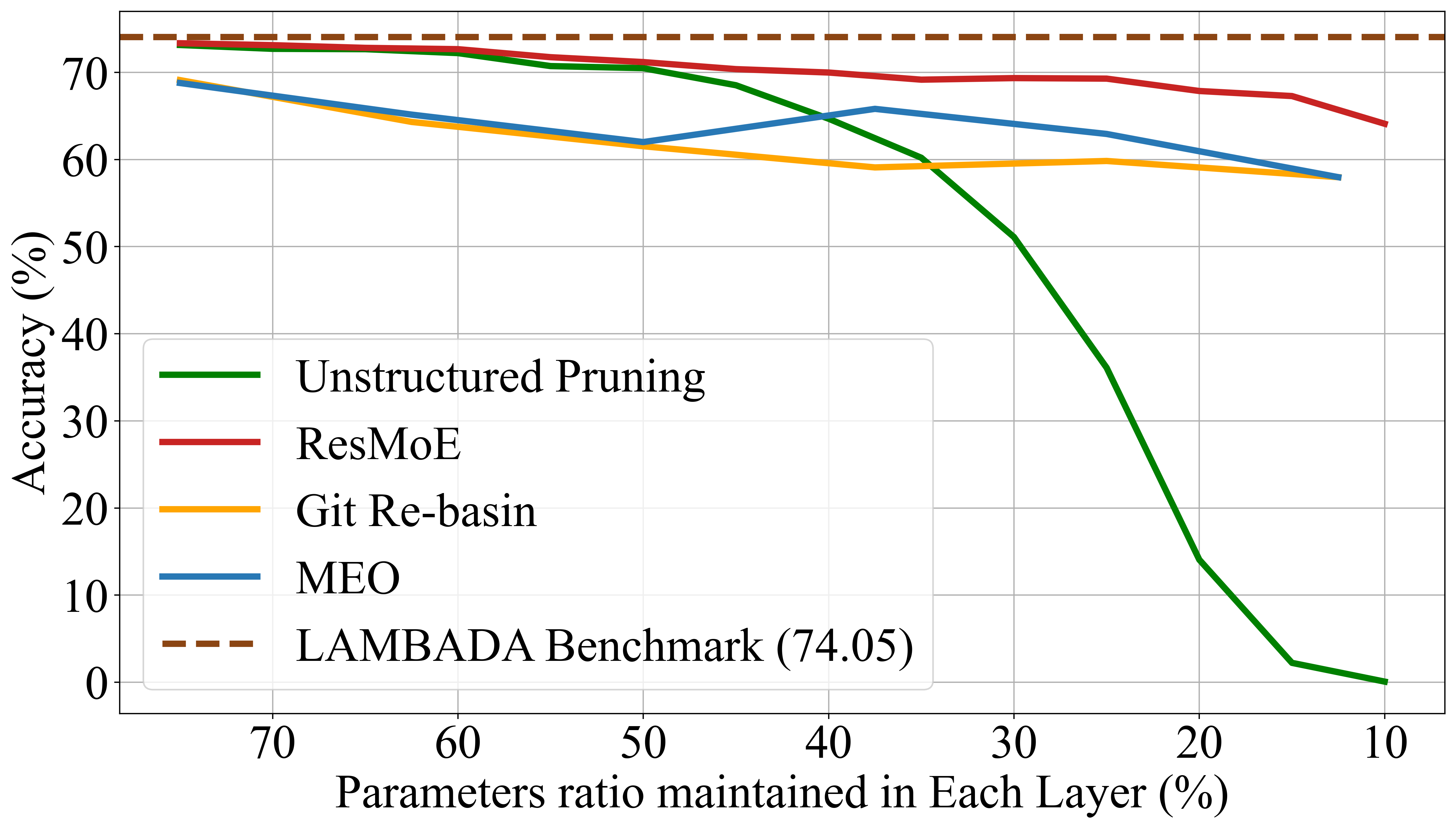}}
\caption{Performance of selected baseline methods on Mixtral w.r.t. various compression rates on the LAMBADA dataset. Note that MEO and Git Re-Basin can only merge experts into at least one so they cannot reach the 10\% compression rate.}
\label{fig:mix}
\end{center}
\vspace{-0.2in}
\end{figure}

\textbf{The effectiveness of Wasserstein barycenter.} In \alg, we choose to compress the residual matrices between the original experts and the barycenter expert. Here, we study the effectiveness of this choice by conducting the vanilla unstructured pruning/SVD without this barycenter expert.

We also conduct the ablation study on the choice of using optimal transport to calculate the barycenter expert. We compare our barycenter with Git Re-Basin \citep{ainsworth2023git} center and with the average center. Considering the generally better performance of unstructured pruning, we conduct these two ablations under unstructured pruning.
The difference between the ablation here and \cref{tab:switch,tab:mixtral} for Git Re-Basin is, during the ablation study, we merge the 8 experts into one expert to obtain the center expert, and follow the framework of \alg~to prune the residuals. While in \cref{tab:switch,tab:mixtral}, we follow the setting of \citet{ainsworth2023git,li2023merge} to merge 8 experts into 2 experts, as our compression ratio is set to 25\%.
Note that we do not contain OT fusion \citep{singh2020model} here, since our experiment on calculating the barycenter in their layer-by-layer form takes more than 4 days to complete, empirically supports our state that the layer-by-layer strategy introduces computational overhead to the process. When performing the algorithm, we observe that Git Re-Basin returns the average center for most of the layers for Mixtral, likely because the dimension of Mixtral reaches a high level (4,096 for the hidden dimension and 14,336 for the inner dimension), and this method is not scalable for such a large model. The outcomes in terms of output performance in Table \ref{tab:ablation_rslt}
clearly demonstrate the beneficial impact of incorporating the barycenter expert.

\textbf{The impact of compression rate.} In the main experiment, we set the compression rate to retain 25\% of the parameter counts. Additionally, we explored the impact of adjusting this rate to different levels. Figure \ref{fig:mix} provides the results of Mixtral on the LAMBADA dataset. Remarkably, with the compression rate set to 10\%, \alg~(UP) manages to achieve results that are not only comparable but even surpass those of baseline methods set at a 30\% compression rate.

\textbf{The scalability of our method.} Our main experiments are conducted on Switch Transformer (switch-base-8) and Mixtral, both with 8 experts per layer. To test the scalability of \alg, we conduct additional experiments on switch-base-16 and DeepseekMoE (64 experts per layer), to verify its ability to maintain performance with an increased number of experts.

Following the same fine-tuning settings as those used for switch-base-8, detailed in Appendix \ref{sec:finetune}, we limited our testing of switch-base-16 to the MRPC dataset due to the constraints of time-intensive supervised fine-tuning. Despite this limitation, Table \ref{tab:switch-16} allows us to draw a similar conclusion as with switch-base-8, that \alg~consistently demonstrates impressive results, affirming its efficacy in maintaining model accuracy with more experts per layer. Additionally, akin to the observations from switch-base-8, we notice that the choice between pruning or applying SVD to the weights, whether concatenated or separated, significantly influences the outcomes. This consistency across different model scales reinforces the impact of these compression techniques on the model's performance, highlighting the nuanced balance between efficiency and accuracy in model optimization. Due to the page limit, the details of DeepseekMoE can be found in \cref{sec:deep}. Additionally, we provide the adaptability of \alg~with expert parallelism and tensor parallelism in \cref{app:parallel}.

\begin{table}[t!]
\centering
\caption{Evaluation Results of Switch Transformer (switch-base-16), with 16 experts per layer. UP stands for Unstructured Pruning, and SP stands for Structured Pruning. We use `concat' to denote concatenate and `sep' to denote separate.}
\label{tab:switch-16}
\begin{tabular}{lcccc}
\toprule
& MRPC \\
\midrule
Switch Transformer  & 90.03$\pm$0.45 \\
\midrule
UP (concat)  &89.47$\pm$0.01 \\
UP (sep)  & 88.48$\pm$0.75\\
SP (concat) &88.40$\pm$0.94\\
SP (sep) &87.34$\pm$0.46\\
SVD (concat) &88.48$\pm$0.62 \\
SVD (sep) &88.48$\pm$0.62 \\
M-SMoE &88.89$\pm$0.59 \\
MEO  &88.51$\pm$0.93 \\
MLP Fusion  &87.91$\pm$0.90 \\
\midrule
\alg~(UP)&\textbf{89.62$\pm$0.01} \\
\bottomrule
\end{tabular}
\end{table}

\section{Conclusions and Limitations}

In this paper, we propose \alg, a data-agnostic MoE model approximation framework that reduces the memory usage of MoE LLMs without retraining. Instead of directly compressing the experts, we turn to approximating the residuals between the Wasserstein barycenter and the original experts. We prove the effectiveness of our method through comprehensive experiments on various backbone models, including Switch Transformer (with an encoder-decoder architecture) and the decoder-only Mixtral and DeepSeekMoE. With \alg, we reduce the counts of parameters by up to 75\% with both successful preservation of the original weight matrices and minimal performance loss in the downstream tasks. The future direction of this work can be the exploration of adopting different compression rates for each layer or even each expert (as experimented in \citet{sharma2023truth} and MC-SMoE~\citep{li2023merge}), or further combining our method with hardware quantization methods.

\textbf{Limitations}. While we have illustrated the success of ResMoE, it is also crucial to understand the limitations that arise in more complex settings: 1) Although producing impressive results, the space efficiency of storing the sparse matrices obtained from unstructured pruning is limited as detailed in \Cref{sec:memory}. 2) \alg~is currently applied during the model inference stage. 
The resulting performance of applying it to fine-tuning is an open question that requires further investigation.

\begin{acks}
This work is supported by National Science Foundation under Award No. IIS-2117902, and Agriculture and Food Research Initiative (AFRI) grant no. 2020-67021-32799/project accession no.1024178 from the USDA National Institute of Food and Agriculture. The views and conclusions are those of the authors and should not be interpreted as representing the official policies of the funding agencies or the government.
\end{acks}

\clearpage

\bibliographystyle{ACM-Reference-Format}
\bibliography{moe}


\begin{thebibliography}{71}


\ifx \showCODEN    \undefined \def \showCODEN     #1{\unskip}     \fi
\ifx \showDOI      \undefined \def \showDOI       #1{#1}\fi
\ifx \showISBNx    \undefined \def \showISBNx     #1{\unskip}     \fi
\ifx \showISBNxiii \undefined \def \showISBNxiii  #1{\unskip}     \fi
\ifx \showISSN     \undefined \def \showISSN      #1{\unskip}     \fi
\ifx \showLCCN     \undefined \def \showLCCN      #1{\unskip}     \fi
\ifx \shownote     \undefined \def \shownote      #1{#1}          \fi
\ifx \showarticletitle \undefined \def \showarticletitle #1{#1}   \fi
\ifx \showURL      \undefined \def \showURL       {\relax}        \fi
\providecommand\bibfield[2]{#2}
\providecommand\bibinfo[2]{#2}
\providecommand\natexlab[1]{#1}
\providecommand\showeprint[2][]{arXiv:#2}

\bibitem[Ai et~al\mbox{.}(2025)]%
        {ai2025mlpfusionefficientfinetuning}
\bibfield{author}{\bibinfo{person}{Mengting Ai}, \bibinfo{person}{Tianxin Wei}, \bibinfo{person}{Yifan Chen}, \bibinfo{person}{Zeming Guo}, {and} \bibinfo{person}{Jingrui He}.} \bibinfo{year}{2025}\natexlab{}.
\newblock \bibinfo{title}{MLP Fusion: Towards Efficient Fine-tuning of Dense and Mixture-of-Experts Language Models}.
\newblock
\newblock
\showeprint[arxiv]{2307.08941}~[cs.LG]
\urldef\tempurl%
\url{https://arxiv.org/abs/2307.08941}
\showURL{%
\tempurl}


\bibitem[Ainsworth et~al\mbox{.}(2023)]%
        {ainsworth2023git}
\bibfield{author}{\bibinfo{person}{Samuel Ainsworth}, \bibinfo{person}{Jonathan Hayase}, {and} \bibinfo{person}{Siddhartha Srinivasa}.} \bibinfo{year}{2023}\natexlab{}.
\newblock \showarticletitle{Git Re-Basin: Merging Models modulo Permutation Symmetries}. In \bibinfo{booktitle}{\emph{The Eleventh International Conference on Learning Representations}}.
\newblock
\urldef\tempurl%
\url{https://openreview.net/forum?id=CQsmMYmlP5T}
\showURL{%
\tempurl}


\bibitem[Anil et~al\mbox{.}(2023)]%
        {anil2023palm}
\bibfield{author}{\bibinfo{person}{Rohan Anil}, \bibinfo{person}{Andrew~M Dai}, \bibinfo{person}{Orhan Firat}, \bibinfo{person}{Melvin Johnson}, \bibinfo{person}{Dmitry Lepikhin}, \bibinfo{person}{Alexandre Passos}, \bibinfo{person}{Siamak Shakeri}, \bibinfo{person}{Emanuel Taropa}, \bibinfo{person}{Paige Bailey}, \bibinfo{person}{Zhifeng Chen}, {et~al\mbox{.}}} \bibinfo{year}{2023}\natexlab{}.
\newblock \showarticletitle{Palm 2 technical report}.
\newblock \bibinfo{journal}{\emph{arXiv preprint arXiv:2305.10403}} (\bibinfo{year}{2023}).
\newblock


\bibitem[Bhandare et~al\mbox{.}(2019)]%
        {bhandare2019efficient}
\bibfield{author}{\bibinfo{person}{Aishwarya Bhandare}, \bibinfo{person}{Vamsi Sripathi}, \bibinfo{person}{Deepthi Karkada}, \bibinfo{person}{Vivek Menon}, \bibinfo{person}{Sun Choi}, \bibinfo{person}{Kushal Datta}, {and} \bibinfo{person}{Vikram Saletore}.} \bibinfo{year}{2019}\natexlab{}.
\newblock \bibinfo{title}{Efficient 8-Bit Quantization of Transformer Neural Machine Language Translation Model}.
\newblock
\newblock
\showeprint[arxiv]{1906.00532}~[cs.LG]


\bibitem[Bisk et~al\mbox{.}(2020)]%
        {bisk2020piqa}
\bibfield{author}{\bibinfo{person}{Yonatan Bisk}, \bibinfo{person}{Rowan Zellers}, \bibinfo{person}{Jianfeng Gao}, \bibinfo{person}{Yejin Choi}, {et~al\mbox{.}}} \bibinfo{year}{2020}\natexlab{}.
\newblock \showarticletitle{PIQA: Reasoning about Physical Commonsense in Natural Language}. In \bibinfo{booktitle}{\emph{Proceedings of the AAAI Conference on Artificial Intelligence}}, Vol.~\bibinfo{volume}{34}. \bibinfo{pages}{7432--7439}.
\newblock


\bibitem[Blalock et~al\mbox{.}(2020)]%
        {blalock2020stateneuralnetworkpruning}
\bibfield{author}{\bibinfo{person}{Davis Blalock}, \bibinfo{person}{Jose Javier~Gonzalez Ortiz}, \bibinfo{person}{Jonathan Frankle}, {and} \bibinfo{person}{John Guttag}.} \bibinfo{year}{2020}\natexlab{}.
\newblock \bibinfo{title}{What is the State of Neural Network Pruning?}
\newblock
\newblock
\showeprint[arxiv]{2003.03033}~[cs.LG]
\urldef\tempurl%
\url{https://arxiv.org/abs/2003.03033}
\showURL{%
\tempurl}


\bibitem[Chen et~al\mbox{.}(2022)]%
        {10.1145/3534678.3539452}
\bibfield{author}{\bibinfo{person}{Yankai Chen}, \bibinfo{person}{Huifeng Guo}, \bibinfo{person}{Yingxue Zhang}, \bibinfo{person}{Chen Ma}, \bibinfo{person}{Ruiming Tang}, \bibinfo{person}{Jingjie Li}, {and} \bibinfo{person}{Irwin King}.} \bibinfo{year}{2022}\natexlab{}.
\newblock \showarticletitle{Learning Binarized Graph Representations with Multi-faceted Quantization Reinforcement for Top-K Recommendation}. In \bibinfo{booktitle}{\emph{Proceedings of the 28th ACM SIGKDD Conference on Knowledge Discovery and Data Mining}} (Washington DC, USA) \emph{(\bibinfo{series}{KDD '22})}. \bibinfo{publisher}{Association for Computing Machinery}, \bibinfo{address}{New York, NY, USA}, \bibinfo{pages}{168–178}.
\newblock
\showISBNx{9781450393850}
\urldef\tempurl%
\url{https://doi.org/10.1145/3534678.3539452}
\showDOI{\tempurl}


\bibitem[Cohen(2016)]%
        {cohen2016nearly}
\bibfield{author}{\bibinfo{person}{Michael~B Cohen}.} \bibinfo{year}{2016}\natexlab{}.
\newblock \showarticletitle{Nearly tight oblivious subspace embeddings by trace inequalities}. In \bibinfo{booktitle}{\emph{Proceedings of the twenty-seventh annual ACM-SIAM symposium on Discrete algorithms}}. SIAM, \bibinfo{pages}{278--287}.
\newblock


\bibitem[Cuturi and Doucet(2014)]%
        {cuturi2014fast}
\bibfield{author}{\bibinfo{person}{Marco Cuturi} {and} \bibinfo{person}{Arnaud Doucet}.} \bibinfo{year}{2014}\natexlab{}.
\newblock \showarticletitle{Fast computation of Wasserstein barycenters}. In \bibinfo{booktitle}{\emph{International conference on machine learning}}. PMLR, \bibinfo{pages}{685--693}.
\newblock


\bibitem[Dai et~al\mbox{.}(2024)]%
        {dai2024deepseekmoe}
\bibfield{author}{\bibinfo{person}{Damai Dai}, \bibinfo{person}{Chengqi Deng}, \bibinfo{person}{Chenggang Zhao}, \bibinfo{person}{R.~X. Xu}, \bibinfo{person}{Huazuo Gao}, \bibinfo{person}{Deli Chen}, \bibinfo{person}{Jiashi Li}, \bibinfo{person}{Wangding Zeng}, \bibinfo{person}{Xingkai Yu}, \bibinfo{person}{Y. Wu}, \bibinfo{person}{Zhenda Xie}, \bibinfo{person}{Y.~K. Li}, \bibinfo{person}{Panpan Huang}, \bibinfo{person}{Fuli Luo}, \bibinfo{person}{Chong Ruan}, \bibinfo{person}{Zhifang Sui}, {and} \bibinfo{person}{Wenfeng Liang}.} \bibinfo{year}{2024}\natexlab{}.
\newblock \bibinfo{title}{DeepSeekMoE: Towards Ultimate Expert Specialization in Mixture-of-Experts Language Models}.
\newblock
\newblock
\showeprint[arxiv]{2401.06066}~[cs.CL]


\bibitem[Denton et~al\mbox{.}(2014a)]%
        {NIPS2014_2afe4567}
\bibfield{author}{\bibinfo{person}{Emily~L Denton}, \bibinfo{person}{Wojciech Zaremba}, \bibinfo{person}{Joan Bruna}, \bibinfo{person}{Yann LeCun}, {and} \bibinfo{person}{Rob Fergus}.} \bibinfo{year}{2014}\natexlab{a}.
\newblock \showarticletitle{Exploiting Linear Structure Within Convolutional Networks for Efficient Evaluation}. In \bibinfo{booktitle}{\emph{Advances in Neural Information Processing Systems}}, \bibfield{editor}{\bibinfo{person}{Z.~Ghahramani}, \bibinfo{person}{M.~Welling}, \bibinfo{person}{C.~Cortes}, \bibinfo{person}{N.~Lawrence}, {and} \bibinfo{person}{K.Q. Weinberger}} (Eds.), Vol.~\bibinfo{volume}{27}. \bibinfo{publisher}{Curran Associates, Inc.}
\newblock
\urldef\tempurl%
\url{https://proceedings.neurips.cc/paper_files/paper/2014/file/2afe4567e1bf64d32a5527244d104cea-Paper.pdf}
\showURL{%
\tempurl}


\bibitem[Denton et~al\mbox{.}(2014b)]%
        {denton2014exploiting}
\bibfield{author}{\bibinfo{person}{Emily~L Denton}, \bibinfo{person}{Wojciech Zaremba}, \bibinfo{person}{Joan Bruna}, \bibinfo{person}{Yann LeCun}, {and} \bibinfo{person}{Rob Fergus}.} \bibinfo{year}{2014}\natexlab{b}.
\newblock \showarticletitle{Exploiting linear structure within convolutional networks for efficient evaluation}.
\newblock \bibinfo{journal}{\emph{Advances in neural information processing systems}}  \bibinfo{volume}{27} (\bibinfo{year}{2014}).
\newblock


\bibitem[Dettmers et~al\mbox{.}(2022)]%
        {dettmers2022llmint8}
\bibfield{author}{\bibinfo{person}{Tim Dettmers}, \bibinfo{person}{Mike Lewis}, \bibinfo{person}{Younes Belkada}, {and} \bibinfo{person}{Luke Zettlemoyer}.} \bibinfo{year}{2022}\natexlab{}.
\newblock \bibinfo{title}{LLM.int8(): 8-bit Matrix Multiplication for Transformers at Scale}.
\newblock
\newblock
\showeprint[arxiv]{2208.07339}~[cs.LG]


\bibitem[Devlin et~al\mbox{.}(2019)]%
        {devlin2019bert}
\bibfield{author}{\bibinfo{person}{Jacob Devlin}, \bibinfo{person}{Ming-Wei Chang}, \bibinfo{person}{Kenton Lee}, {and} \bibinfo{person}{Kristina Toutanova}.} \bibinfo{year}{2019}\natexlab{}.
\newblock \bibinfo{title}{BERT: Pre-training of Deep Bidirectional Transformers for Language Understanding}.
\newblock
\newblock
\showeprint[arxiv]{1810.04805}~[cs.CL]


\bibitem[Dolan and Brockett(2005)]%
        {dolan-brockett-2005-automatically}
\bibfield{author}{\bibinfo{person}{William~B. Dolan} {and} \bibinfo{person}{Chris Brockett}.} \bibinfo{year}{2005}\natexlab{}.
\newblock \showarticletitle{Automatically Constructing a Corpus of Sentential Paraphrases}. In \bibinfo{booktitle}{\emph{Proceedings of the Third International Workshop on Paraphrasing ({IWP}2005)}}.
\newblock
\urldef\tempurl%
\url{https://aclanthology.org/I05-5002}
\showURL{%
\tempurl}


\bibitem[Dou et~al\mbox{.}(2023)]%
        {dou2023loramoe}
\bibfield{author}{\bibinfo{person}{Shihan Dou}, \bibinfo{person}{Enyu Zhou}, \bibinfo{person}{Yan Liu}, \bibinfo{person}{Songyang Gao}, \bibinfo{person}{Jun Zhao}, \bibinfo{person}{Wei Shen}, \bibinfo{person}{Yuhao Zhou}, \bibinfo{person}{Zhiheng Xi}, \bibinfo{person}{Xiao Wang}, \bibinfo{person}{Xiaoran Fan}, {et~al\mbox{.}}} \bibinfo{year}{2023}\natexlab{}.
\newblock \showarticletitle{LORAMOE: REVOLUTIONIZING MIXTURE OF EX-PERTS FOR MAINTAINING WORLD KNOWLEDGE IN LANGUAGE MODEL ALIGNMENT}.
\newblock \bibinfo{journal}{\emph{arXiv preprint arXiv:2312.09979}} (\bibinfo{year}{2023}).
\newblock


\bibitem[Fan et~al\mbox{.}(2021)]%
        {Fan_2021_ICCV}
\bibfield{author}{\bibinfo{person}{Haoqi Fan}, \bibinfo{person}{Bo Xiong}, \bibinfo{person}{Karttikeya Mangalam}, \bibinfo{person}{Yanghao Li}, \bibinfo{person}{Zhicheng Yan}, \bibinfo{person}{Jitendra Malik}, {and} \bibinfo{person}{Christoph Feichtenhofer}.} \bibinfo{year}{2021}\natexlab{}.
\newblock \showarticletitle{Multiscale Vision Transformers}. In \bibinfo{booktitle}{\emph{Proceedings of the IEEE/CVF International Conference on Computer Vision (ICCV)}}. \bibinfo{pages}{6824--6835}.
\newblock


\bibitem[Fedus et~al\mbox{.}(2022)]%
        {fedus2022switch}
\bibfield{author}{\bibinfo{person}{William Fedus}, \bibinfo{person}{Barret Zoph}, {and} \bibinfo{person}{Noam Shazeer}.} \bibinfo{year}{2022}\natexlab{}.
\newblock \showarticletitle{Switch transformers: Scaling to trillion parameter models with simple and efficient sparsity}.
\newblock \bibinfo{journal}{\emph{The Journal of Machine Learning Research}} \bibinfo{volume}{23}, \bibinfo{number}{1} (\bibinfo{year}{2022}), \bibinfo{pages}{5232--5270}.
\newblock


\bibitem[Frankle and Carbin(2018)]%
        {frankle2018lottery}
\bibfield{author}{\bibinfo{person}{Jonathan Frankle} {and} \bibinfo{person}{Michael Carbin}.} \bibinfo{year}{2018}\natexlab{}.
\newblock \showarticletitle{The lottery ticket hypothesis: Finding sparse, trainable neural networks}.
\newblock \bibinfo{journal}{\emph{arXiv preprint arXiv:1803.03635}} (\bibinfo{year}{2018}).
\newblock


\bibitem[Frankle et~al\mbox{.}(2020)]%
        {10.5555/3524938.3525243}
\bibfield{author}{\bibinfo{person}{Jonathan Frankle}, \bibinfo{person}{Gintare~Karolina Dziugaite}, \bibinfo{person}{Daniel~M. Roy}, {and} \bibinfo{person}{Michael Carbin}.} \bibinfo{year}{2020}\natexlab{}.
\newblock \showarticletitle{Linear mode connectivity and the lottery ticket hypothesis}. In \bibinfo{booktitle}{\emph{Proceedings of the 37th International Conference on Machine Learning}} \emph{(\bibinfo{series}{ICML'20})}. \bibinfo{publisher}{JMLR.org}, Article \bibinfo{articleno}{305}, \bibinfo{numpages}{11}~pages.
\newblock


\bibitem[Frantar et~al\mbox{.}(2023)]%
        {frantar2023gptq}
\bibfield{author}{\bibinfo{person}{Elias Frantar}, \bibinfo{person}{Saleh Ashkboos}, \bibinfo{person}{Torsten Hoefler}, {and} \bibinfo{person}{Dan Alistarh}.} \bibinfo{year}{2023}\natexlab{}.
\newblock \bibinfo{title}{GPTQ: Accurate Post-Training Quantization for Generative Pre-trained Transformers}.
\newblock
\newblock
\showeprint[arxiv]{2210.17323}~[cs.LG]


\bibitem[Gao et~al\mbox{.}(2022)]%
        {gao-etal-2022-parameter}
\bibfield{author}{\bibinfo{person}{Ze-Feng Gao}, \bibinfo{person}{Peiyu Liu}, \bibinfo{person}{Wayne~Xin Zhao}, \bibinfo{person}{Zhong-Yi Lu}, {and} \bibinfo{person}{Ji-Rong Wen}.} \bibinfo{year}{2022}\natexlab{}.
\newblock \showarticletitle{Parameter-Efficient Mixture-of-Experts Architecture for Pre-trained Language Models}. In \bibinfo{booktitle}{\emph{Proceedings of the 29th International Conference on Computational Linguistics}}, \bibfield{editor}{\bibinfo{person}{Nicoletta Calzolari}, \bibinfo{person}{Chu-Ren Huang}, \bibinfo{person}{Hansaem Kim}, \bibinfo{person}{James Pustejovsky}, \bibinfo{person}{Leo Wanner}, \bibinfo{person}{Key-Sun Choi}, \bibinfo{person}{Pum-Mo Ryu}, \bibinfo{person}{Hsin-Hsi Chen}, \bibinfo{person}{Lucia Donatelli}, \bibinfo{person}{Heng Ji}, \bibinfo{person}{Sadao Kurohashi}, \bibinfo{person}{Patrizia Paggio}, \bibinfo{person}{Nianwen Xue}, \bibinfo{person}{Seokhwan Kim}, \bibinfo{person}{Younggyun Hahm}, \bibinfo{person}{Zhong He}, \bibinfo{person}{Tony~Kyungil Lee}, \bibinfo{person}{Enrico Santus}, \bibinfo{person}{Francis Bond}, {and} \bibinfo{person}{Seung-Hoon Na}} (Eds.). \bibinfo{publisher}{International Committee on Computational Linguistics}, \bibinfo{address}{Gyeongju, Republic
  of Korea}, \bibinfo{pages}{3263--3273}.
\newblock
\urldef\tempurl%
\url{https://aclanthology.org/2022.coling-1.288}
\showURL{%
\tempurl}


\bibitem[Gou et~al\mbox{.}(2021)]%
        {gou2021knowledge}
\bibfield{author}{\bibinfo{person}{Jianping Gou}, \bibinfo{person}{Baosheng Yu}, \bibinfo{person}{Stephen~J Maybank}, {and} \bibinfo{person}{Dacheng Tao}.} \bibinfo{year}{2021}\natexlab{}.
\newblock \showarticletitle{Knowledge distillation: A survey}.
\newblock \bibinfo{journal}{\emph{International Journal of Computer Vision}}  \bibinfo{volume}{129} (\bibinfo{year}{2021}), \bibinfo{pages}{1789--1819}.
\newblock


\bibitem[Han et~al\mbox{.}(2015a)]%
        {han2015learning}
\bibfield{author}{\bibinfo{person}{Song Han}, \bibinfo{person}{Jeff Pool}, \bibinfo{person}{John Tran}, {and} \bibinfo{person}{William Dally}.} \bibinfo{year}{2015}\natexlab{a}.
\newblock \showarticletitle{Learning both weights and connections for efficient neural network}.
\newblock \bibinfo{journal}{\emph{Advances in neural information processing systems}}  \bibinfo{volume}{28} (\bibinfo{year}{2015}).
\newblock


\bibitem[Han et~al\mbox{.}(2015b)]%
        {NIPS2015_ae0eb3ee}
\bibfield{author}{\bibinfo{person}{Song Han}, \bibinfo{person}{Jeff Pool}, \bibinfo{person}{John Tran}, {and} \bibinfo{person}{William Dally}.} \bibinfo{year}{2015}\natexlab{b}.
\newblock \showarticletitle{Learning both Weights and Connections for Efficient Neural Network}. In \bibinfo{booktitle}{\emph{Advances in Neural Information Processing Systems}}, \bibfield{editor}{\bibinfo{person}{C.~Cortes}, \bibinfo{person}{N.~Lawrence}, \bibinfo{person}{D.~Lee}, \bibinfo{person}{M.~Sugiyama}, {and} \bibinfo{person}{R.~Garnett}} (Eds.), Vol.~\bibinfo{volume}{28}. \bibinfo{publisher}{Curran Associates, Inc.}
\newblock
\urldef\tempurl%
\url{https://proceedings.neurips.cc/paper_files/paper/2015/file/ae0eb3eed39d2bcef4622b2499a05fe6-Paper.pdf}
\showURL{%
\tempurl}


\bibitem[He et~al\mbox{.}(2023a)]%
        {he2023preserving}
\bibfield{author}{\bibinfo{person}{Guande He}, \bibinfo{person}{Jianfei Chen}, {and} \bibinfo{person}{Jun Zhu}.} \bibinfo{year}{2023}\natexlab{a}.
\newblock \showarticletitle{Preserving Pre-trained Features Helps Calibrate Fine-tuned Language Models}. In \bibinfo{booktitle}{\emph{The Eleventh International Conference on Learning Representations}}.
\newblock
\urldef\tempurl%
\url{https://openreview.net/forum?id=NI7StoWHJPT}
\showURL{%
\tempurl}


\bibitem[He et~al\mbox{.}(2022)]%
        {10.1145/3534678.3539315}
\bibfield{author}{\bibinfo{person}{Huarui He}, \bibinfo{person}{Jie Wang}, \bibinfo{person}{Zhanqiu Zhang}, {and} \bibinfo{person}{Feng Wu}.} \bibinfo{year}{2022}\natexlab{}.
\newblock \showarticletitle{Compressing Deep Graph Neural Networks via Adversarial Knowledge Distillation}. In \bibinfo{booktitle}{\emph{Proceedings of the 28th ACM SIGKDD Conference on Knowledge Discovery and Data Mining}} (Washington DC, USA) \emph{(\bibinfo{series}{KDD '22})}. \bibinfo{publisher}{Association for Computing Machinery}, \bibinfo{address}{New York, NY, USA}, \bibinfo{pages}{534–544}.
\newblock
\showISBNx{9781450393850}
\urldef\tempurl%
\url{https://doi.org/10.1145/3534678.3539315}
\showDOI{\tempurl}


\bibitem[He et~al\mbox{.}(2023b)]%
        {he2023merging}
\bibfield{author}{\bibinfo{person}{Shwai He}, \bibinfo{person}{Run-Ze Fan}, \bibinfo{person}{Liang Ding}, \bibinfo{person}{Li Shen}, \bibinfo{person}{Tianyi Zhou}, {and} \bibinfo{person}{Dacheng Tao}.} \bibinfo{year}{2023}\natexlab{b}.
\newblock \showarticletitle{Merging Experts into One: Improving Computational Efficiency of Mixture of Experts}. In \bibinfo{booktitle}{\emph{Proceedings of the 2023 Conference on Empirical Methods in Natural Language Processing}}. \bibinfo{pages}{14685--14691}.
\newblock


\bibitem[Jiang et~al\mbox{.}(2023)]%
        {jiang2023mistral}
\bibfield{author}{\bibinfo{person}{Albert~Q. Jiang}, \bibinfo{person}{Alexandre Sablayrolles}, \bibinfo{person}{Arthur Mensch}, \bibinfo{person}{Chris Bamford}, \bibinfo{person}{Devendra~Singh Chaplot}, \bibinfo{person}{Diego de~las Casas}, \bibinfo{person}{Florian Bressand}, \bibinfo{person}{Gianna Lengyel}, \bibinfo{person}{Guillaume Lample}, \bibinfo{person}{Lucile Saulnier}, \bibinfo{person}{Lélio~Renard Lavaud}, \bibinfo{person}{Marie-Anne Lachaux}, \bibinfo{person}{Pierre Stock}, \bibinfo{person}{Teven~Le Scao}, \bibinfo{person}{Thibaut Lavril}, \bibinfo{person}{Thomas Wang}, \bibinfo{person}{Timothée Lacroix}, {and} \bibinfo{person}{William~El Sayed}.} \bibinfo{year}{2023}\natexlab{}.
\newblock \bibinfo{title}{Mistral 7B}.
\newblock
\newblock
\showeprint[arxiv]{2310.06825}~[cs.CL]


\bibitem[Jiang et~al\mbox{.}(2024)]%
        {jiang2024mixtral}
\bibfield{author}{\bibinfo{person}{Albert~Q. Jiang}, \bibinfo{person}{Alexandre Sablayrolles}, \bibinfo{person}{Antoine Roux}, \bibinfo{person}{Arthur Mensch}, \bibinfo{person}{Blanche Savary}, \bibinfo{person}{Chris Bamford}, \bibinfo{person}{Devendra~Singh Chaplot}, \bibinfo{person}{Diego de~las Casas}, \bibinfo{person}{Emma~Bou Hanna}, \bibinfo{person}{Florian Bressand}, \bibinfo{person}{Gianna Lengyel}, \bibinfo{person}{Guillaume Bour}, \bibinfo{person}{Guillaume Lample}, \bibinfo{person}{Lélio~Renard Lavaud}, \bibinfo{person}{Lucile Saulnier}, \bibinfo{person}{Marie-Anne Lachaux}, \bibinfo{person}{Pierre Stock}, \bibinfo{person}{Sandeep Subramanian}, \bibinfo{person}{Sophia Yang}, \bibinfo{person}{Szymon Antoniak}, \bibinfo{person}{Teven~Le Scao}, \bibinfo{person}{Théophile Gervet}, \bibinfo{person}{Thibaut Lavril}, \bibinfo{person}{Thomas Wang}, \bibinfo{person}{Timothée Lacroix}, {and} \bibinfo{person}{William~El Sayed}.} \bibinfo{year}{2024}\natexlab{}.
\newblock \bibinfo{title}{Mixtral of Experts}.
\newblock
\newblock
\showeprint[arxiv]{2401.04088}~[cs.LG]


\bibitem[Kang et~al\mbox{.}(2020)]%
        {10.1145/3340531.3412005}
\bibfield{author}{\bibinfo{person}{SeongKu Kang}, \bibinfo{person}{Junyoung Hwang}, \bibinfo{person}{Wonbin Kweon}, {and} \bibinfo{person}{Hwanjo Yu}.} \bibinfo{year}{2020}\natexlab{}.
\newblock \showarticletitle{DE-RRD: A Knowledge Distillation Framework for Recommender System}. In \bibinfo{booktitle}{\emph{Proceedings of the 29th ACM International Conference on Information \& Knowledge Management}} (Virtual Event, Ireland) \emph{(\bibinfo{series}{CIKM '20})}. \bibinfo{publisher}{Association for Computing Machinery}, \bibinfo{address}{New York, NY, USA}, \bibinfo{pages}{605–614}.
\newblock
\showISBNx{9781450368599}
\urldef\tempurl%
\url{https://doi.org/10.1145/3340531.3412005}
\showDOI{\tempurl}


\bibitem[Kong et~al\mbox{.}(2023)]%
        {kong2023swapmoe}
\bibfield{author}{\bibinfo{person}{Rui Kong}, \bibinfo{person}{Yuanchun Li}, \bibinfo{person}{Qingtian Feng}, \bibinfo{person}{Weijun Wang}, \bibinfo{person}{Linghe Kong}, {and} \bibinfo{person}{Yunxin Liu}.} \bibinfo{year}{2023}\natexlab{}.
\newblock \bibinfo{title}{SwapMoE: Efficient Memory-Constrained Serving of Large Sparse MoE Models via Dynamic Expert Pruning and Swapping}.
\newblock
\newblock
\showeprint[arxiv]{2308.15030}~[cs.AI]


\bibitem[Lee et~al\mbox{.}(2019)]%
        {lee2018snip}
\bibfield{author}{\bibinfo{person}{Namhoon Lee}, \bibinfo{person}{Thalaiyasingam Ajanthan}, {and} \bibinfo{person}{Philip Torr}.} \bibinfo{year}{2019}\natexlab{}.
\newblock \showarticletitle{{SNIP}: {SINGLE}-{SHOT} {NETWORK} {PRUNING} {BASED} {ON} {CONNECTION} {SENSITIVITY}}. In \bibinfo{booktitle}{\emph{International Conference on Learning Representations}}.
\newblock
\urldef\tempurl%
\url{https://openreview.net/forum?id=B1VZqjAcYX}
\showURL{%
\tempurl}


\bibitem[Li et~al\mbox{.}(2023b)]%
        {10.1145/3580305.3599284}
\bibfield{author}{\bibinfo{person}{Junyan Li}, \bibinfo{person}{Li~Lyna Zhang}, \bibinfo{person}{Jiahang Xu}, \bibinfo{person}{Yujing Wang}, \bibinfo{person}{Shaoguang Yan}, \bibinfo{person}{Yunqing Xia}, \bibinfo{person}{Yuqing Yang}, \bibinfo{person}{Ting Cao}, \bibinfo{person}{Hao Sun}, \bibinfo{person}{Weiwei Deng}, \bibinfo{person}{Qi Zhang}, {and} \bibinfo{person}{Mao Yang}.} \bibinfo{year}{2023}\natexlab{b}.
\newblock \showarticletitle{Constraint-aware and Ranking-distilled Token Pruning for Efficient Transformer Inference}. In \bibinfo{booktitle}{\emph{Proceedings of the 29th ACM SIGKDD Conference on Knowledge Discovery and Data Mining}} (Long Beach, CA, USA) \emph{(\bibinfo{series}{KDD '23})}. \bibinfo{publisher}{Association for Computing Machinery}, \bibinfo{address}{New York, NY, USA}, \bibinfo{pages}{1280–1290}.
\newblock
\showISBNx{9798400701030}
\urldef\tempurl%
\url{https://doi.org/10.1145/3580305.3599284}
\showDOI{\tempurl}


\bibitem[Li et~al\mbox{.}(2023c)]%
        {li2023merge}
\bibfield{author}{\bibinfo{person}{Pingzhi Li}, \bibinfo{person}{Zhenyu Zhang}, \bibinfo{person}{Prateek Yadav}, \bibinfo{person}{Yi-Lin Sung}, \bibinfo{person}{Yu Cheng}, \bibinfo{person}{Mohit Bansal}, {and} \bibinfo{person}{Tianlong Chen}.} \bibinfo{year}{2023}\natexlab{c}.
\newblock \showarticletitle{Merge, Then Compress: Demystify Efficient SMoE with Hints from Its Routing Policy}.
\newblock \bibinfo{journal}{\emph{arXiv preprint arXiv:2310.01334}} (\bibinfo{year}{2023}).
\newblock


\bibitem[Li et~al\mbox{.}(2023a)]%
        {li2023losparse}
\bibfield{author}{\bibinfo{person}{Yixiao Li}, \bibinfo{person}{Yifan Yu}, \bibinfo{person}{Qingru Zhang}, \bibinfo{person}{Chen Liang}, \bibinfo{person}{Pengcheng He}, \bibinfo{person}{Weizhu Chen}, {and} \bibinfo{person}{Tuo Zhao}.} \bibinfo{year}{2023}\natexlab{a}.
\newblock \bibinfo{title}{LoSparse: Structured Compression of Large Language Models based on Low-Rank and Sparse Approximation}.
\newblock
\newblock
\showeprint[arxiv]{2306.11222}~[cs.LG]


\bibitem[Li et~al\mbox{.}(2022)]%
        {li-etal-2022-dq}
\bibfield{author}{\bibinfo{person}{Zheng Li}, \bibinfo{person}{Zijian Wang}, \bibinfo{person}{Ming Tan}, \bibinfo{person}{Ramesh Nallapati}, \bibinfo{person}{Parminder Bhatia}, \bibinfo{person}{Andrew Arnold}, \bibinfo{person}{Bing Xiang}, {and} \bibinfo{person}{Dan Roth}.} \bibinfo{year}{2022}\natexlab{}.
\newblock \showarticletitle{{DQ}-{BART}: Efficient Sequence-to-Sequence Model via Joint Distillation and Quantization}. In \bibinfo{booktitle}{\emph{Proceedings of the 60th Annual Meeting of the Association for Computational Linguistics (Volume 2: Short Papers)}}. \bibinfo{publisher}{Association for Computational Linguistics}, \bibinfo{address}{Dublin, Ireland}, \bibinfo{pages}{203--211}.
\newblock
\urldef\tempurl%
\url{https://doi.org/10.18653/v1/2022.acl-short.22}
\showDOI{\tempurl}


\bibitem[Liu et~al\mbox{.}(2022)]%
        {pmlr-v162-liu22k}
\bibfield{author}{\bibinfo{person}{Chang Liu}, \bibinfo{person}{Chenfei Lou}, \bibinfo{person}{Runzhong Wang}, \bibinfo{person}{Alan~Yuhan Xi}, \bibinfo{person}{Li Shen}, {and} \bibinfo{person}{Junchi Yan}.} \bibinfo{year}{2022}\natexlab{}.
\newblock \showarticletitle{Deep Neural Network Fusion via Graph Matching with Applications to Model Ensemble and Federated Learning}. In \bibinfo{booktitle}{\emph{Proceedings of the 39th International Conference on Machine Learning}} \emph{(\bibinfo{series}{Proceedings of Machine Learning Research}, Vol.~\bibinfo{volume}{162})}, \bibfield{editor}{\bibinfo{person}{Kamalika Chaudhuri}, \bibinfo{person}{Stefanie Jegelka}, \bibinfo{person}{Le~Song}, \bibinfo{person}{Csaba Szepesvari}, \bibinfo{person}{Gang Niu}, {and} \bibinfo{person}{Sivan Sabato}} (Eds.). \bibinfo{publisher}{PMLR}, \bibinfo{pages}{13857--13869}.
\newblock
\urldef\tempurl%
\url{https://proceedings.mlr.press/v162/liu22k.html}
\showURL{%
\tempurl}


\bibitem[Liu et~al\mbox{.}(2021)]%
        {liu2021swin}
\bibfield{author}{\bibinfo{person}{Ze Liu}, \bibinfo{person}{Yutong Lin}, \bibinfo{person}{Yue Cao}, \bibinfo{person}{Han Hu}, \bibinfo{person}{Yixuan Wei}, \bibinfo{person}{Zheng Zhang}, \bibinfo{person}{Stephen Lin}, {and} \bibinfo{person}{Baining Guo}.} \bibinfo{year}{2021}\natexlab{}.
\newblock \showarticletitle{Swin transformer: Hierarchical vision transformer using shifted windows}. In \bibinfo{booktitle}{\emph{Proceedings of the IEEE/CVF international conference on computer vision}}. \bibinfo{pages}{10012--10022}.
\newblock


\bibitem[Liu et~al\mbox{.}(2018)]%
        {liu2018rethinking}
\bibfield{author}{\bibinfo{person}{Zhuang Liu}, \bibinfo{person}{Mingjie Sun}, \bibinfo{person}{Tinghui Zhou}, \bibinfo{person}{Gao Huang}, {and} \bibinfo{person}{Trevor Darrell}.} \bibinfo{year}{2018}\natexlab{}.
\newblock \showarticletitle{Rethinking the value of network pruning}.
\newblock \bibinfo{journal}{\emph{arXiv preprint arXiv:1810.05270}} (\bibinfo{year}{2018}).
\newblock


\bibitem[Lu et~al\mbox{.}(2024)]%
        {lu2024expertsequalefficientexpert}
\bibfield{author}{\bibinfo{person}{Xudong Lu}, \bibinfo{person}{Qi Liu}, \bibinfo{person}{Yuhui Xu}, \bibinfo{person}{Aojun Zhou}, \bibinfo{person}{Siyuan Huang}, \bibinfo{person}{Bo Zhang}, \bibinfo{person}{Junchi Yan}, {and} \bibinfo{person}{Hongsheng Li}.} \bibinfo{year}{2024}\natexlab{}.
\newblock \bibinfo{title}{Not All Experts are Equal: Efficient Expert Pruning and Skipping for Mixture-of-Experts Large Language Models}.
\newblock
\newblock
\showeprint[arxiv]{2402.14800}~[cs.CL]
\urldef\tempurl%
\url{https://arxiv.org/abs/2402.14800}
\showURL{%
\tempurl}


\bibitem[Merity et~al\mbox{.}(2016)]%
        {merity2016pointer}
\bibfield{author}{\bibinfo{person}{Stephen Merity}, \bibinfo{person}{Caiming Xiong}, \bibinfo{person}{James Bradbury}, {and} \bibinfo{person}{Richard Socher}.} \bibinfo{year}{2016}\natexlab{}.
\newblock \showarticletitle{Pointer sentinel mixture models}.
\newblock \bibinfo{journal}{\emph{arXiv preprint arXiv:1609.07843}} (\bibinfo{year}{2016}).
\newblock


\bibitem[Mukhoti et~al\mbox{.}(2023)]%
        {mukhoti2023finetuning}
\bibfield{author}{\bibinfo{person}{Jishnu Mukhoti}, \bibinfo{person}{Yarin Gal}, \bibinfo{person}{Philip H.~S. Torr}, {and} \bibinfo{person}{Puneet~K. Dokania}.} \bibinfo{year}{2023}\natexlab{}.
\newblock \bibinfo{title}{Fine-tuning can cripple your foundation model; preserving features may be the solution}.
\newblock
\newblock
\showeprint[arxiv]{2308.13320}~[cs.LG]


\bibitem[Muzio et~al\mbox{.}(2024)]%
        {muzio2024seermoesparseexpertefficiency}
\bibfield{author}{\bibinfo{person}{Alexandre Muzio}, \bibinfo{person}{Alex Sun}, {and} \bibinfo{person}{Churan He}.} \bibinfo{year}{2024}\natexlab{}.
\newblock \bibinfo{title}{SEER-MoE: Sparse Expert Efficiency through Regularization for Mixture-of-Experts}.
\newblock
\newblock
\showeprint[arxiv]{2404.05089}~[cs.CL]
\urldef\tempurl%
\url{https://arxiv.org/abs/2404.05089}
\showURL{%
\tempurl}


\bibitem[OpenAI(2022)]%
        {OpenAI22}
\bibfield{author}{\bibinfo{person}{OpenAI}.} \bibinfo{year}{2022}\natexlab{}.
\newblock \bibinfo{title}{Techniques for training large neural networks}.
\newblock
\newblock
\urldef\tempurl%
\url{https://openai.com/research/techniques-for-training-large-neural-networks}
\showURL{%
\tempurl}


\bibitem[Paperno et~al\mbox{.}(2016)]%
        {paperno-etal-2016-lambada}
\bibfield{author}{\bibinfo{person}{Denis Paperno}, \bibinfo{person}{Germ{\'a}n Kruszewski}, \bibinfo{person}{Angeliki Lazaridou}, \bibinfo{person}{Ngoc~Quan Pham}, \bibinfo{person}{Raffaella Bernardi}, \bibinfo{person}{Sandro Pezzelle}, \bibinfo{person}{Marco Baroni}, \bibinfo{person}{Gemma Boleda}, {and} \bibinfo{person}{Raquel Fern{\'a}ndez}.} \bibinfo{year}{2016}\natexlab{}.
\newblock \showarticletitle{The {LAMBADA} dataset: Word prediction requiring a broad discourse context}. In \bibinfo{booktitle}{\emph{Proceedings of the 54th Annual Meeting of the Association for Computational Linguistics (Volume 1: Long Papers)}}, \bibfield{editor}{\bibinfo{person}{Katrin Erk} {and} \bibinfo{person}{Noah~A. Smith}} (Eds.). \bibinfo{publisher}{Association for Computational Linguistics}, \bibinfo{address}{Berlin, Germany}, \bibinfo{pages}{1525--1534}.
\newblock
\urldef\tempurl%
\url{https://doi.org/10.18653/v1/P16-1144}
\showDOI{\tempurl}


\bibitem[Peyre and Cuturi(2020)]%
        {peyr2020computational}
\bibfield{author}{\bibinfo{person}{Gabriel Peyre} {and} \bibinfo{person}{Marco Cuturi}.} \bibinfo{year}{2020}\natexlab{}.
\newblock \bibinfo{title}{Computational Optimal Transport}.
\newblock
\newblock
\showeprint[arxiv]{1803.00567}~[stat.ML]


\bibitem[Raffel et~al\mbox{.}(2023a)]%
        {raffel2023exploring}
\bibfield{author}{\bibinfo{person}{Colin Raffel}, \bibinfo{person}{Noam Shazeer}, \bibinfo{person}{Adam Roberts}, \bibinfo{person}{Katherine Lee}, \bibinfo{person}{Sharan Narang}, \bibinfo{person}{Michael Matena}, \bibinfo{person}{Yanqi Zhou}, \bibinfo{person}{Wei Li}, {and} \bibinfo{person}{Peter~J. Liu}.} \bibinfo{year}{2023}\natexlab{a}.
\newblock \bibinfo{title}{Exploring the Limits of Transfer Learning with a Unified Text-to-Text Transformer}.
\newblock
\newblock
\showeprint[arxiv]{1910.10683}~[cs.LG]


\bibitem[Raffel et~al\mbox{.}(2023b)]%
        {raffel2023exploringlimitstransferlearning}
\bibfield{author}{\bibinfo{person}{Colin Raffel}, \bibinfo{person}{Noam Shazeer}, \bibinfo{person}{Adam Roberts}, \bibinfo{person}{Katherine Lee}, \bibinfo{person}{Sharan Narang}, \bibinfo{person}{Michael Matena}, \bibinfo{person}{Yanqi Zhou}, \bibinfo{person}{Wei Li}, {and} \bibinfo{person}{Peter~J. Liu}.} \bibinfo{year}{2023}\natexlab{b}.
\newblock \bibinfo{title}{Exploring the Limits of Transfer Learning with a Unified Text-to-Text Transformer}.
\newblock
\newblock
\showeprint[arxiv]{1910.10683}~[cs.LG]
\urldef\tempurl%
\url{https://arxiv.org/abs/1910.10683}
\showURL{%
\tempurl}


\bibitem[Sakaguchi et~al\mbox{.}(2021)]%
        {sakaguchi2021winogrande}
\bibfield{author}{\bibinfo{person}{Keisuke Sakaguchi}, \bibinfo{person}{Ronan~Le Bras}, \bibinfo{person}{Chandra Bhagavatula}, {and} \bibinfo{person}{Yejin Choi}.} \bibinfo{year}{2021}\natexlab{}.
\newblock \showarticletitle{Winogrande: An adversarial winograd schema challenge at scale}.
\newblock \bibinfo{journal}{\emph{Commun. ACM}} \bibinfo{volume}{64}, \bibinfo{number}{9} (\bibinfo{year}{2021}), \bibinfo{pages}{99--106}.
\newblock


\bibitem[Sharma et~al\mbox{.}(2023)]%
        {sharma2023truth}
\bibfield{author}{\bibinfo{person}{Pratyusha Sharma}, \bibinfo{person}{Jordan~T. Ash}, {and} \bibinfo{person}{Dipendra Misra}.} \bibinfo{year}{2023}\natexlab{}.
\newblock \bibinfo{title}{The Truth is in There: Improving Reasoning in Language Models with Layer-Selective Rank Reduction}.
\newblock
\newblock
\showeprint[arxiv]{2312.13558}~[cs.LG]


\bibitem[Shazeer et~al\mbox{.}(2017)]%
        {shazeer2017outrageously}
\bibfield{author}{\bibinfo{person}{Noam Shazeer}, \bibinfo{person}{Azalia Mirhoseini}, \bibinfo{person}{Krzysztof Maziarz}, \bibinfo{person}{Andy Davis}, \bibinfo{person}{Quoc Le}, \bibinfo{person}{Geoffrey Hinton}, {and} \bibinfo{person}{Jeff Dean}.} \bibinfo{year}{2017}\natexlab{}.
\newblock \bibinfo{title}{Outrageously Large Neural Networks: The Sparsely-Gated Mixture-of-Experts Layer}.
\newblock
\newblock
\showeprint[arxiv]{1701.06538}~[cs.LG]


\bibitem[Shoeybi et~al\mbox{.}(2020)]%
        {shoeybi2020megatronlm}
\bibfield{author}{\bibinfo{person}{Mohammad Shoeybi}, \bibinfo{person}{Mostofa Patwary}, \bibinfo{person}{Raul Puri}, \bibinfo{person}{Patrick LeGresley}, \bibinfo{person}{Jared Casper}, {and} \bibinfo{person}{Bryan Catanzaro}.} \bibinfo{year}{2020}\natexlab{}.
\newblock \bibinfo{title}{Megatron-LM: Training Multi-Billion Parameter Language Models Using Model Parallelism}.
\newblock
\newblock
\showeprint[arxiv]{1909.08053}~[cs.CL]


\bibitem[Singh and Jaggi(2020)]%
        {singh2020model}
\bibfield{author}{\bibinfo{person}{Sidak~Pal Singh} {and} \bibinfo{person}{Martin Jaggi}.} \bibinfo{year}{2020}\natexlab{}.
\newblock \showarticletitle{Model fusion via optimal transport}.
\newblock \bibinfo{journal}{\emph{Advances in Neural Information Processing Systems}}  \bibinfo{volume}{33} (\bibinfo{year}{2020}), \bibinfo{pages}{22045--22055}.
\newblock


\bibitem[Socher et~al\mbox{.}(2013)]%
        {socher2013recursive}
\bibfield{author}{\bibinfo{person}{Richard Socher}, \bibinfo{person}{Alex Perelygin}, \bibinfo{person}{Jean Wu}, \bibinfo{person}{Jason Chuang}, \bibinfo{person}{Christopher~D Manning}, \bibinfo{person}{Andrew~Y Ng}, {and} \bibinfo{person}{Christopher Potts}.} \bibinfo{year}{2013}\natexlab{}.
\newblock \showarticletitle{Recursive deep models for semantic compositionality over a sentiment treebank}. In \bibinfo{booktitle}{\emph{Proceedings of the 2013 conference on empirical methods in natural language processing}}. \bibinfo{pages}{1631--1642}.
\newblock


\bibitem[Stoica et~al\mbox{.}(2024)]%
        {stoica2024zipitmergingmodelsdifferent}
\bibfield{author}{\bibinfo{person}{George Stoica}, \bibinfo{person}{Daniel Bolya}, \bibinfo{person}{Jakob Bjorner}, \bibinfo{person}{Pratik Ramesh}, \bibinfo{person}{Taylor Hearn}, {and} \bibinfo{person}{Judy Hoffman}.} \bibinfo{year}{2024}\natexlab{}.
\newblock \bibinfo{title}{ZipIt! Merging Models from Different Tasks without Training}.
\newblock
\newblock
\showeprint[arxiv]{2305.03053}~[cs.CV]
\urldef\tempurl%
\url{https://arxiv.org/abs/2305.03053}
\showURL{%
\tempurl}


\bibitem[Sun et~al\mbox{.}(2024)]%
        {sun2024a}
\bibfield{author}{\bibinfo{person}{Mingjie Sun}, \bibinfo{person}{Zhuang Liu}, \bibinfo{person}{Anna Bair}, {and} \bibinfo{person}{J~Zico Kolter}.} \bibinfo{year}{2024}\natexlab{}.
\newblock \showarticletitle{A Simple and Effective Pruning Approach for Large Language Models}. In \bibinfo{booktitle}{\emph{The Twelfth International Conference on Learning Representations}}.
\newblock
\urldef\tempurl%
\url{https://openreview.net/forum?id=PxoFut3dWW}
\showURL{%
\tempurl}


\bibitem[Tanaka et~al\mbox{.}(2020)]%
        {NEURIPS2020_46a4378f}
\bibfield{author}{\bibinfo{person}{Hidenori Tanaka}, \bibinfo{person}{Daniel Kunin}, \bibinfo{person}{Daniel~L Yamins}, {and} \bibinfo{person}{Surya Ganguli}.} \bibinfo{year}{2020}\natexlab{}.
\newblock \showarticletitle{Pruning neural networks without any data by iteratively conserving synaptic flow}. In \bibinfo{booktitle}{\emph{Advances in Neural Information Processing Systems}}, \bibfield{editor}{\bibinfo{person}{H.~Larochelle}, \bibinfo{person}{M.~Ranzato}, \bibinfo{person}{R.~Hadsell}, \bibinfo{person}{M.F. Balcan}, {and} \bibinfo{person}{H.~Lin}} (Eds.), Vol.~\bibinfo{volume}{33}. \bibinfo{publisher}{Curran Associates, Inc.}, \bibinfo{pages}{6377--6389}.
\newblock
\urldef\tempurl%
\url{https://proceedings.neurips.cc/paper_files/paper/2020/file/46a4378f835dc8040c8057beb6a2da52-Paper.pdf}
\showURL{%
\tempurl}


\bibitem[Tao et~al\mbox{.}(2022)]%
        {tao-etal-2022-compression}
\bibfield{author}{\bibinfo{person}{Chaofan Tao}, \bibinfo{person}{Lu Hou}, \bibinfo{person}{Wei Zhang}, \bibinfo{person}{Lifeng Shang}, \bibinfo{person}{Xin Jiang}, \bibinfo{person}{Qun Liu}, \bibinfo{person}{Ping Luo}, {and} \bibinfo{person}{Ngai Wong}.} \bibinfo{year}{2022}\natexlab{}.
\newblock \showarticletitle{Compression of Generative Pre-trained Language Models via Quantization}. In \bibinfo{booktitle}{\emph{Proceedings of the 60th Annual Meeting of the Association for Computational Linguistics (Volume 1: Long Papers)}}. \bibinfo{publisher}{Association for Computational Linguistics}, \bibinfo{address}{Dublin, Ireland}, \bibinfo{pages}{4821--4836}.
\newblock
\urldef\tempurl%
\url{https://doi.org/10.18653/v1/2022.acl-long.331}
\showDOI{\tempurl}


\bibitem[Touvron et~al\mbox{.}(2023)]%
        {touvron2023llama}
\bibfield{author}{\bibinfo{person}{Hugo Touvron}, \bibinfo{person}{Louis Martin}, \bibinfo{person}{Kevin Stone}, \bibinfo{person}{Peter Albert}, \bibinfo{person}{Amjad Almahairi}, \bibinfo{person}{Yasmine Babaei}, \bibinfo{person}{Nikolay Bashlykov}, \bibinfo{person}{Soumya Batra}, \bibinfo{person}{Prajjwal Bhargava}, \bibinfo{person}{Shruti Bhosale}, \bibinfo{person}{Dan Bikel}, \bibinfo{person}{Lukas Blecher}, \bibinfo{person}{Cristian~Canton Ferrer}, \bibinfo{person}{Moya Chen}, \bibinfo{person}{Guillem Cucurull}, \bibinfo{person}{David Esiobu}, \bibinfo{person}{Jude Fernandes}, \bibinfo{person}{Jeremy Fu}, \bibinfo{person}{Wenyin Fu}, \bibinfo{person}{Brian Fuller}, \bibinfo{person}{Cynthia Gao}, \bibinfo{person}{Vedanuj Goswami}, \bibinfo{person}{Naman Goyal}, \bibinfo{person}{Anthony Hartshorn}, \bibinfo{person}{Saghar Hosseini}, \bibinfo{person}{Rui Hou}, \bibinfo{person}{Hakan Inan}, \bibinfo{person}{Marcin Kardas}, \bibinfo{person}{Viktor Kerkez}, \bibinfo{person}{Madian Khabsa},
  \bibinfo{person}{Isabel Kloumann}, \bibinfo{person}{Artem Korenev}, \bibinfo{person}{Punit~Singh Koura}, \bibinfo{person}{Marie-Anne Lachaux}, \bibinfo{person}{Thibaut Lavril}, \bibinfo{person}{Jenya Lee}, \bibinfo{person}{Diana Liskovich}, \bibinfo{person}{Yinghai Lu}, \bibinfo{person}{Yuning Mao}, \bibinfo{person}{Xavier Martinet}, \bibinfo{person}{Todor Mihaylov}, \bibinfo{person}{Pushkar Mishra}, \bibinfo{person}{Igor Molybog}, \bibinfo{person}{Yixin Nie}, \bibinfo{person}{Andrew Poulton}, \bibinfo{person}{Jeremy Reizenstein}, \bibinfo{person}{Rashi Rungta}, \bibinfo{person}{Kalyan Saladi}, \bibinfo{person}{Alan Schelten}, \bibinfo{person}{Ruan Silva}, \bibinfo{person}{Eric~Michael Smith}, \bibinfo{person}{Ranjan Subramanian}, \bibinfo{person}{Xiaoqing~Ellen Tan}, \bibinfo{person}{Binh Tang}, \bibinfo{person}{Ross Taylor}, \bibinfo{person}{Adina Williams}, \bibinfo{person}{Jian~Xiang Kuan}, \bibinfo{person}{Puxin Xu}, \bibinfo{person}{Zheng Yan}, \bibinfo{person}{Iliyan Zarov}, \bibinfo{person}{Yuchen
  Zhang}, \bibinfo{person}{Angela Fan}, \bibinfo{person}{Melanie Kambadur}, \bibinfo{person}{Sharan Narang}, \bibinfo{person}{Aurelien Rodriguez}, \bibinfo{person}{Robert Stojnic}, \bibinfo{person}{Sergey Edunov}, {and} \bibinfo{person}{Thomas Scialom}.} \bibinfo{year}{2023}\natexlab{}.
\newblock \bibinfo{title}{Llama 2: Open Foundation and Fine-Tuned Chat Models}.
\newblock
\newblock
\showeprint[arxiv]{2307.09288}~[cs.CL]


\bibitem[Vaswani et~al\mbox{.}(2017)]%
        {NIPS2017_3f5ee243}
\bibfield{author}{\bibinfo{person}{Ashish Vaswani}, \bibinfo{person}{Noam Shazeer}, \bibinfo{person}{Niki Parmar}, \bibinfo{person}{Jakob Uszkoreit}, \bibinfo{person}{Llion Jones}, \bibinfo{person}{Aidan~N Gomez}, \bibinfo{person}{\L~ukasz Kaiser}, {and} \bibinfo{person}{Illia Polosukhin}.} \bibinfo{year}{2017}\natexlab{}.
\newblock \showarticletitle{Attention is All you Need}. In \bibinfo{booktitle}{\emph{Advances in Neural Information Processing Systems}}, \bibfield{editor}{\bibinfo{person}{I.~Guyon}, \bibinfo{person}{U.~Von Luxburg}, \bibinfo{person}{S.~Bengio}, \bibinfo{person}{H.~Wallach}, \bibinfo{person}{R.~Fergus}, \bibinfo{person}{S.~Vishwanathan}, {and} \bibinfo{person}{R.~Garnett}} (Eds.), Vol.~\bibinfo{volume}{30}. \bibinfo{publisher}{Curran Associates, Inc.}
\newblock
\urldef\tempurl%
\url{https://proceedings.neurips.cc/paper_files/paper/2017/file/3f5ee243547dee91fbd053c1c4a845aa-Paper.pdf}
\showURL{%
\tempurl}


\bibitem[Wang et~al\mbox{.}(2022)]%
        {wang2022exploring}
\bibfield{author}{\bibinfo{person}{Benyou Wang}, \bibinfo{person}{Yuxin Ren}, \bibinfo{person}{Lifeng Shang}, \bibinfo{person}{Xin Jiang}, {and} \bibinfo{person}{Qun Liu}.} \bibinfo{year}{2022}\natexlab{}.
\newblock \showarticletitle{Exploring extreme parameter compression for pre-trained language models}. In \bibinfo{booktitle}{\emph{International Conference on Learning Representations}}.
\newblock
\urldef\tempurl%
\url{https://openreview.net/forum?id=RftryyYyjiG}
\showURL{%
\tempurl}


\bibitem[Wang et~al\mbox{.}(2020)]%
        {wang2020picking}
\bibfield{author}{\bibinfo{person}{Chaoqi Wang}, \bibinfo{person}{Guodong Zhang}, {and} \bibinfo{person}{Roger Grosse}.} \bibinfo{year}{2020}\natexlab{}.
\newblock \showarticletitle{Picking winning tickets before training by preserving gradient flow}.
\newblock \bibinfo{journal}{\emph{arXiv preprint arXiv:2002.07376}} (\bibinfo{year}{2020}).
\newblock


\bibitem[Wang et~al\mbox{.}(2021)]%
        {Wang_2021_ICCV}
\bibfield{author}{\bibinfo{person}{Wenhai Wang}, \bibinfo{person}{Enze Xie}, \bibinfo{person}{Xiang Li}, \bibinfo{person}{Deng-Ping Fan}, \bibinfo{person}{Kaitao Song}, \bibinfo{person}{Ding Liang}, \bibinfo{person}{Tong Lu}, \bibinfo{person}{Ping Luo}, {and} \bibinfo{person}{Ling Shao}.} \bibinfo{year}{2021}\natexlab{}.
\newblock \showarticletitle{Pyramid Vision Transformer: A Versatile Backbone for Dense Prediction Without Convolutions}. In \bibinfo{booktitle}{\emph{Proceedings of the IEEE/CVF International Conference on Computer Vision (ICCV)}}. \bibinfo{pages}{568--578}.
\newblock


\bibitem[Warstadt et~al\mbox{.}(2019)]%
        {warstadt-etal-2019-neural}
\bibfield{author}{\bibinfo{person}{Alex Warstadt}, \bibinfo{person}{Amanpreet Singh}, {and} \bibinfo{person}{Samuel~R. Bowman}.} \bibinfo{year}{2019}\natexlab{}.
\newblock \showarticletitle{Neural Network Acceptability Judgments}.
\newblock \bibinfo{journal}{\emph{Transactions of the Association for Computational Linguistics}}  \bibinfo{volume}{7} (\bibinfo{year}{2019}), \bibinfo{pages}{625--641}.
\newblock
\urldef\tempurl%
\url{https://doi.org/10.1162/tacl_a_00290}
\showDOI{\tempurl}


\bibitem[Wei et~al\mbox{.}(2023)]%
        {chen2023ntk}
\bibfield{author}{\bibinfo{person}{Tianxin Wei}, \bibinfo{person}{Zeming Guo}, \bibinfo{person}{Yifan Chen}, {and} \bibinfo{person}{Jingrui He}.} \bibinfo{year}{2023}\natexlab{}.
\newblock \showarticletitle{{NTK}-approximating {MLP} Fusion for Efficient Language Model Fine-tuning}. In \bibinfo{booktitle}{\emph{Proceedings of the 40th International Conference on Machine Learning}} \emph{(\bibinfo{series}{Proceedings of Machine Learning Research}, Vol.~\bibinfo{volume}{202})}, \bibfield{editor}{\bibinfo{person}{Andreas Krause}, \bibinfo{person}{Emma Brunskill}, \bibinfo{person}{Kyunghyun Cho}, \bibinfo{person}{Barbara Engelhardt}, \bibinfo{person}{Sivan Sabato}, {and} \bibinfo{person}{Jonathan Scarlett}} (Eds.). \bibinfo{publisher}{PMLR}, \bibinfo{pages}{36821--36838}.
\newblock
\urldef\tempurl%
\url{https://proceedings.mlr.press/v202/wei23b.html}
\showURL{%
\tempurl}


\bibitem[Williams et~al\mbox{.}(2018)]%
        {williams-etal-2018-broad}
\bibfield{author}{\bibinfo{person}{Adina Williams}, \bibinfo{person}{Nikita Nangia}, {and} \bibinfo{person}{Samuel Bowman}.} \bibinfo{year}{2018}\natexlab{}.
\newblock \showarticletitle{A Broad-Coverage Challenge Corpus for Sentence Understanding through Inference}. In \bibinfo{booktitle}{\emph{Proceedings of the 2018 Conference of the North {A}merican Chapter of the Association for Computational Linguistics: Human Language Technologies, Volume 1 (Long Papers)}}, \bibfield{editor}{\bibinfo{person}{Marilyn Walker}, \bibinfo{person}{Heng Ji}, {and} \bibinfo{person}{Amanda Stent}} (Eds.). \bibinfo{publisher}{Association for Computational Linguistics}, \bibinfo{address}{New Orleans, Louisiana}, \bibinfo{pages}{1112--1122}.
\newblock
\urldef\tempurl%
\url{https://doi.org/10.18653/v1/N18-1101}
\showDOI{\tempurl}


\bibitem[Wolf et~al\mbox{.}(2020)]%
        {wolf-etal-2020-transformers}
\bibfield{author}{\bibinfo{person}{Thomas Wolf}, \bibinfo{person}{Lysandre Debut}, \bibinfo{person}{Victor Sanh}, \bibinfo{person}{Julien Chaumond}, \bibinfo{person}{Clement Delangue}, \bibinfo{person}{Anthony Moi}, \bibinfo{person}{Pierric Cistac}, \bibinfo{person}{Tim Rault}, \bibinfo{person}{Rémi Louf}, \bibinfo{person}{Morgan Funtowicz}, \bibinfo{person}{Joe Davison}, \bibinfo{person}{Sam Shleifer}, \bibinfo{person}{Patrick von Platen}, \bibinfo{person}{Clara Ma}, \bibinfo{person}{Yacine Jernite}, \bibinfo{person}{Julien Plu}, \bibinfo{person}{Canwen Xu}, \bibinfo{person}{Teven~Le Scao}, \bibinfo{person}{Sylvain Gugger}, \bibinfo{person}{Mariama Drame}, \bibinfo{person}{Quentin Lhoest}, {and} \bibinfo{person}{Alexander~M. Rush}.} \bibinfo{year}{2020}\natexlab{}.
\newblock \showarticletitle{Transformers: State-of-the-Art Natural Language Processing}. In \bibinfo{booktitle}{\emph{Proceedings of the 2020 Conference on Empirical Methods in Natural Language Processing: System Demonstrations}}. \bibinfo{publisher}{Association for Computational Linguistics}, \bibinfo{address}{Online}, \bibinfo{pages}{38--45}.
\newblock
\urldef\tempurl%
\url{https://www.aclweb.org/anthology/2020.emnlp-demos.6}
\showURL{%
\tempurl}


\bibitem[Xue et~al\mbox{.}(2022)]%
        {xue2022student}
\bibfield{author}{\bibinfo{person}{Fuzhao Xue}, \bibinfo{person}{Xiaoxin He}, \bibinfo{person}{Xiaozhe Ren}, \bibinfo{person}{Yuxuan Lou}, {and} \bibinfo{person}{Yang You}.} \bibinfo{year}{2022}\natexlab{}.
\newblock \bibinfo{title}{One Student Knows All Experts Know: From Sparse to Dense}.
\newblock
\newblock
\showeprint[arxiv]{2201.10890}~[cs.LG]


\bibitem[Yuan et~al\mbox{.}(2022)]%
        {yuan2022distributed}
\bibfield{author}{\bibinfo{person}{Binhang Yuan}, \bibinfo{person}{Cameron~R Wolfe}, \bibinfo{person}{Chen Dun}, \bibinfo{person}{Yuxin Tang}, \bibinfo{person}{Anastasios Kyrillidis}, {and} \bibinfo{person}{Chris Jermaine}.} \bibinfo{year}{2022}\natexlab{}.
\newblock \showarticletitle{Distributed learning of fully connected neural networks using independent subnet training}.
\newblock \bibinfo{journal}{\emph{Proceedings of the VLDB Endowment}} \bibinfo{volume}{15}, \bibinfo{number}{8} (\bibinfo{year}{2022}), \bibinfo{pages}{1581--1590}.
\newblock


\bibitem[Zhang et~al\mbox{.}(2023)]%
        {10.1145/3583780.3615499}
\bibfield{author}{\bibinfo{person}{Beichuan Zhang}, \bibinfo{person}{Chenggen Sun}, \bibinfo{person}{Jianchao Tan}, \bibinfo{person}{Xinjun Cai}, \bibinfo{person}{Jun Zhao}, \bibinfo{person}{Mengqi Miao}, \bibinfo{person}{Kang Yin}, \bibinfo{person}{Chengru Song}, \bibinfo{person}{Na Mou}, {and} \bibinfo{person}{Yang Song}.} \bibinfo{year}{2023}\natexlab{}.
\newblock \showarticletitle{SHARK: A Lightweight Model Compression Approach for Large-scale Recommender Systems}. In \bibinfo{booktitle}{\emph{Proceedings of the 32nd ACM International Conference on Information and Knowledge Management}} (Birmingham, United Kingdom) \emph{(\bibinfo{series}{CIKM '23})}. \bibinfo{publisher}{Association for Computing Machinery}, \bibinfo{address}{New York, NY, USA}, \bibinfo{pages}{4930–4937}.
\newblock
\showISBNx{9798400701245}
\urldef\tempurl%
\url{https://doi.org/10.1145/3583780.3615499}
\showDOI{\tempurl}


\end{thebibliography}

\appendix

\newpage
\section{Details for Experiments}

The models are implemented through the `transformers' package developed by \citet{wolf-etal-2020-transformers}. When using Switch Transformer for the GLUE classification tasks, we add a classification head according to the setting of T5. 

\subsection{Details of Experiment Settings}
\label{app:experiment_set}

In the zero-shot context, we adhere to the standard practice for zero-shot LLM evaluation. This involves assessing the pre-trained LLM on various downstream tasks without any additional fine-tuning. In the fine-tuning scenario, we employ the typical procedure of initially fine-tuning the pre-trained model on downstream tasks, followed by evaluation during the inference stage. Our approach is concentrated on enhancing the inference stage, ensuring that after the implementation of our methods, no further fine-tuning is required, and the compressed model will be ready-to-use.
\label{sec:finetune}

\begin{table}[h!]
\centering
\caption{Fine-tuning hyper-parameters setting for Switch Transformer.}
\label{tab:finetune}
\begin{tabular}{lcccc}
\toprule
 & Value  \\
\midrule
Optimizer    & AdamW  \\
Adam $\epsilon$        & 1e-08  \\
Adam $\beta$          & (0.9, 0.98)  \\
warm-up steps & 8\\
weight decay&0.01\\
\bottomrule
\end{tabular}
\end{table}

For Mixtral, we adhere to the model's configuration card, as no hyperparameter tuning is required for zero-shot tasks. For Switch Transformer, we employ AdamW optimizer with a linear warm-up step count of 8. We explore the hyper-parameters settings during the supervised fine-tuning stage.
For Switch Transformer, the learning rate is searched in the range of \{1e-4,2e-4,3e-4,5e-4,1e-3\}, the batch size within the range of \{16,32,64\}, and the training epoch within the range of \{3,5,10,15,20\}. The details of the AdamW optimizer which is fixed for all datasets are given in table \ref{tab:finetune}.

To ensure comparability across methods, we standardize the parameter count reduction for the experts to around 75\%, which means 25\% of the parameters will be retained.
All the methods are performed at the top 24 layers of Mixtral, and the top 8 MoE layers of Switch Transformer. All the methods are applied during the inference stage. 

\subsection{Experiment Results of DeepSeekMoE}
\label{sec:deep}

DeepSeekMoE \citep{dai2024deepseekmoe} features more fine-grained experts compared to alternative structures, with each expert sized at just 8.7M, markedly smaller than Mixtral's 176.2M. It introduces a unique, independent shared expert atop each MoE layer, aiming to encapsulate universal information across layers. This approach is inspired by observations from Mixtral's router analysis, which indicates a roughly equal routing probability for each expert during the inference stage, suggesting the presence of universal knowledge within each expert. Consequently, an additional shared expert is integrated, anticipated to store universal information and thereby enhance the diversity of the remaining experts. In our experiments with DeepSeekMoE, we exclude this shared expert, considering its anticipated information content significantly differs from that of the non-shared experts.

Table \ref{tab:deep} provides the results for DeepSeekMoE. Note that the results for LAMBADA dataset are not included here, since DeepSeekMoE produces extremely bad results for this dataset (0.04\% Accuracy). Even though the experts of DeepSeekMoE are also initialized through copy-and-paste, the existence of the shared expert distinguishes those experts in the MoE layers from each other, leading to the poor performance of the merge methods, which aligns with our proposition that merge methods will potentially impair the generalization ability of the original model.

\begin{table*}[h!]
\centering
\caption{Zero-shot results of DeepSeekMoE. For Pruning, we choose Unstructured Pruning over Structured Pruning based on the observations from the previous experiments that Unstructured Pruning usually has better results on NLG tasks.}
\label{tab:deep}
\begin{tabular}{lcccc}
\toprule
 & WikiText (PPL)\textdownarrow &  PIQA (ACC) & WinoGrande (ACC)  \\
\midrule
DeepSeekMoE    & 6.51   & 78.84 & 68.75  \\
\midrule
Pruning        & 10.46  & 73.12 & 62.83  \\
SVD            & 26.93  & 63.06 & 57.46  \\
M-SMoE &34.76 &62.79 &54.22\\
MEO&33.94&62.35&54.14\\
\midrule
\alg~(UP)   & \textbf{10.39$\pm$0.10} & \textbf{73.39$\pm$0.01} & \textbf{64.35$\pm$0.01} \\
\bottomrule
\end{tabular}
\end{table*}

\subsection{Compression Setting for Each Method}
\label{sec:compress_ratio}

In order to match our setting of the 75\% compression rate, each method has to be tailored differently.
For Pruning, we mask 75\% weights units with the lowest L1-norm within each expert. %
For SVD, the details of calculating the rate of the parameters can be referred to in Appendix \ref{sec:SVD}.
For M-SMoE, Git Re-Basin, and MEO, we reduce the expert count of each MoE layer from 8 to 2. 
For MLP Fusion, we reduce the intermediate dimension to 25\%.
For \alg, we mask 75\% weights units with the lowest L1-norm within each residual matrix. We do not count the overhead storage of the barycenter experts here since we aim to prove the effectiveness of our algorithm, and as the number of experts grows, the redundancy of this overhead will diminish. We provide additional experiments performed at switch-base-16 and DeepSeekMoE (64 experts per layer) \citep{dai2024deepseekmoe} as a support, which can be found in \Cref{tab:switch-16} and \cref{sec:deep}.

\subsection{The Parameter Count of SVD}
\label{sec:SVD}
For SVD, to make the the parameters retained for each expert equal, for Switch Transformer we have:
\begin{align*}
    p_\text{I}\times k+ k + k\times 2p &\approx k\times (p_\text{I}+2p)\\
    s\times p_\text{I}\times 2p &= 2spp_\text{I},
\end{align*}
where $s$ is the parameter rate we retain (25\% here), and $k$ is the number of top-$k$ singular values in SVD. For Switch Transformer, we have $p_\text{I}=4p$, so $k=\dfrac{1}{3}sp_\text{I}$.

For Mixtral:
\begin{align*}
    p_\text{I}\times k+ k + k\times 3p &\approx k\times (p_\text{I}+3p)\\
    s\times p_\text{I}\times 3p &= 3spp_\text{I},
\end{align*}
here we have $p_\text{I}=3.5p$, so $k=\dfrac{6}{13}sp_\text{I}$. 

For DeepSeekMoE:
\begin{align*}
    p_\text{I}\times k+ k + k\times 3p &\approx k\times (p_\text{I}+3p)\\
    s\times p_\text{I}\times 3p &= 3spp_\text{I},
\end{align*}
here we have $p_\text{I}=\dfrac{11}{16}p$, so $k=\dfrac{48}{59}sp_\text{I}$.

\subsection{Evaluation of Approximation Error for MLP Fusion}
\label{sec:approx}

We first reformulate MLP fusion~\citep{ai2025mlpfusionefficientfinetuning} with the notations in this paper. 
In computing the fused MLP for expert $k$, we obtain the centroid weight/bias $\mtx{\wt W}_k = \brkt{\mtx{\wt W}^{(1)}_k, \mtx{\tilde b}^{(1)}_k, (\mtx{\wt W}^{(2)}_k)^T} \in \mb R^{c \times (2p + 1)}$, as well as the one-hot clustering matrix $\mtx C_k \in \mb R^{c \times p_\text{I}}$ indicating how the $p_\text{I}$ neurons are partitioned into $c$ clusters.
MLP fusion then proposes to compute
\begin{align*}
\mtx{\wt W}^{(2)}_k \paren{\mtx C_k \mtx C_k^T} \sigma \paren{\mtx{\wt W}^{(1)}_k \mtx x + \mtx{\tilde b}^{(1)}_k} + \mtx b^{(2)}_k
\end{align*}
as the approximation to the original expert $k$.

\citet{ai2025mlpfusionefficientfinetuning} suggest the expression above is equivalent to replacing $\mtx W_k = \brkt{\mtx W_k^{(1)}, \mtx b_k^{(1)}, (\mtx W_k^{(2)})^T} \in \mb R^{p_\text{I} \times (2p + 1)}$ by $\mtx C_k^T \mtx{\wt W}_k$, 
considering 
\begin{align*}
\mtx{\wt W}^{(2)}_k \paren{\mtx C_k \mtx C_k^T} \sigma \paren{\mtx{\wt W}^{(1)}_k \mtx x + \mtx{\tilde b}^{(1)}_k} + \mtx b^{(2)}_k \\
= \paren{\mtx{\wt W}^{(2)}_k \mtx C_k} \sigma \paren{\mtx C_k^T \mtx{\wt W}^{(1)}_k \mtx x + \mtx C_k^T \mtx{\tilde b}^{(1)}_k} + \mtx b^{(2)}_k.
\end{align*}
We can thus calculate $\norm{ \mtx W_k - \mtx C_k^T \mtx{\wt W}_k }_F^2$ as the approximation error on expert $k$ for MLP fusion.

\subsection{Details of the Datasets}

\label{app:dataset_license}
We provide the details of the datasets we used in the experiment along with their license here. The statistics can be found in \cref{tab:data_class,tab:data_zeroshot}.
\begin{table*}[h!]

    \centering
    \caption{Dataset statistics of fine-tuned classification tasks.}
    \label{tab:data_class}
    \begin{tabular}{lcccc}
    \toprule
       Dataset &Category  & Train size & Test Size & Classes \\
    \midrule   
    SST-2 &Sentiment Analysis &67,349&872&2\\
    MRPC & Paraphrase Identification&3,668&408&2\\
    CoLA&Linguistic Acceptability Judgment &8,551&1,043&2\\
    MNLI& Textual Entailment &392,702&9,815&3\\

    \bottomrule
    \end{tabular}

\end{table*}

\begin{table*}[h!]
    \centering
    \caption{Dataset statistics of zero-shot tasks.}
    \label{tab:data_zeroshot}
    
    \begin{tabular}{lccc}
    \toprule
       Dataset   &Category& Test Size & Average Text Length \\
    \midrule   
    PIQA &Commonsense Reasoning &1,838 &36.08\\
    WikiText &Language Modeling &4,358 &295.00\\
    WinoGrande &Commonsense Reasoning&1,267&100.78\\
    LAMBADA & Text Understanding & 4,896&341.72\\

    \bottomrule
    \end{tabular}
    
\end{table*}

\begin{itemize}
    \item PIQA \citep{bisk2020piqa}: PIQA, or Physical Interaction Question Answering, is a benchmark dataset that assesses AI systems' commonsense reasoning abilities regarding physical knowledge. It challenges models with multiple-choice questions related to everyday physical interactions, testing their understanding of object manipulation and functionality in real-world scenarios. While humans perform well on PIQA, it presents a significant challenge to AI models, making it crucial for advancing AI research, especially in robotics and conversational AI. PIQA is licensed under Academic Free License v3.0.
    \item WikiText\citep{merity2016pointer}: WikiText-103 is a widely-used dataset in Natural Language Processing, ideal for language modeling and text generation tasks. Derived from verified Wikipedia articles, it offers over 100 million tokens of well-structured text, preserving original formatting. Its extensive vocabulary and varied syntax make it a valuable resource for training advanced language models. WikiText-103 serves as a crucial benchmark for evaluating language models' performance in handling real-world textual data, with a license of CC BY-SA 3.0.
    \item WinoGrande\citep{sakaguchi2021winogrande}: Winogrande, an extension of the Winograd Schema Challenge, consists of ambiguous sentence pairs requiring deep language understanding and commonsense reasoning to resolve. It aims to overcome limitations in previous datasets and assesses AI models' ability to comprehend nuanced language and context, making it vital for advancing natural language understanding. Winogrande is licensed under CC-BY.
    \item LAMBADA\citep{paperno-etal-2016-lambada}: LAMBADA is a challenging benchmark designed for evaluating computational models' language understanding, focusing on predicting the final word in a passage. It requires models to grasp broad context and long-range dependencies within text passages. LAMBADA pushes the boundaries of language models, particularly in handling complex, context-dependent linguistic phenomena, making it a valuable tool for advancing natural language processing. It is licensed under CC BY 4.0.
    \item SST-2\citep{socher2013recursive}: SST-2, the Stanford Sentiment Treebank version 2, is a popular dataset for sentiment analysis. It contains movie review sentences labeled as positive or negative, excluding neutral sentences, providing a binary classification task. This dataset is notable for its fine-grained annotation, as it includes sentiment labels for every subphase within the sentence parse trees. SST-2 is widely used for training and evaluating models on sentiment analysis, testing their ability to understand nuanced emotional tones in text, with the license of CC0: Public Domain.
    \item MRPC\citep{dolan-brockett-2005-automatically}: 
The Microsoft Research Paraphrase Corpus (MRPC) evaluates models on paraphrase identification by using sentence pairs from online news sources. MRPC is a part of the GLUE benchmark and is valuable for assessing a model's ability to understand and compare semantic content in sentences, especially in semantic analysis tasks. The license of MRPC is unknown.
    \item CoLA\citep{warstadt-etal-2019-neural}: The Corpus of Linguistic Acceptability (CoLA) assesses models' linguistic acceptability judgment. It distinguishes between grammatically acceptable and unacceptable sentences, emphasizing the importance of grammatical understanding in language comprehension and model evaluation. The license for CoLA is not specified.
    \item MNLI\citep{warstadt-etal-2019-neural}: The Multi-Genre Natural Language Inference (MNLI) dataset is a diverse corpus for natural language understanding tasks, focusing on textual entailment. It includes pairs of sentences and challenges models to determine whether the second sentence entails, contradicts, or remains neutral to the first sentence. MNLI's wide range of genres and diverse content makes it a robust benchmark for evaluating models in natural language inference tasks. Most of the data are under the OANC’s license, with the other falling under several permissive licenses, a Creative Commons Share-Alike 3.0 Unported
License, and  Creative Commons Attribution 3.0 Unported Licenses.
\end{itemize}

\subsection{Implementation Trick}
\label{sec:memory}

In our application of unstructured pruning with \alg, a key consideration is its impact on memory storage. The Pytorch version we use supports only the Coordinate format (COO) for sparse matrices, which necessitates storing indices as int64. For instance, in an MLP of Mixtral, the original memory footprint is 672MB. However, a version pruned to 75\% sparsity consumes more memory, around 840MB, with 672MB used just for storing indices. If the indices could be stored as int16, the memory requirement for the indices would be reduced to 168MB, thus the memory of the entire MLP unit
would be significantly reduced to 336MB. Furthermore, if the index is stored in the format of CSR instead of COO, the memory size can be reduced to 252MB. Although addressing this limitation falls outside our current scope, we aim to explore solutions to this challenge in future work.

On the other hand, when choosing SVD as the compression method, it will be able to directly reduce memory usage since SVD will reduce the size of each matrix. We remark that even though utilizing SVD here leads to slightly worse results than unstructured pruning, it still performs better compared to the baseline methods. Also, when the number of experts goes up, the overhead introduced by the center expert diminishes. 

Based on such settings, we provide the memory information on Mixtral (8 experts per layer) and DeepSeekMoE (64 experts per layer) in \Cref{tab:memory} with the overhead center expert included. 
Also, when the number of experts goes up, the overhead memory introduced by the center expert diminishes.

\begin{table}[t!]
\caption{Memory usage of one MoE layer in Mixtral \& DeepSeekMoE (MB).}
    \centering
    \begin{tabular}{lcc}
     \toprule
& Mixtral  & DeepSeekMoE \\
\midrule
Full & 5,376 & 2,112 \\
\midrule
UP & 2,016 & 1,056 \\
SP & 1,344 & 528 \\
SVD & 1,344 & 528 \\
M-SMoE & 1,344 & 528 \\
Git Re-Basin & 1,344 & 528 \\
MEO & 1,344 & 528 \\
MLP Fusion & 1,344 & 528 \\
\midrule
ResMoE (UP) & 2,688 & 1,089 \\
ResMoE (SVD) & 2,016 & 561 \\
\bottomrule
    \end{tabular}
    
    \label{tab:memory}
\end{table}

\subsection{Efficiency Evaluation}
\label{app:efficiency}

In addition to the memory information, we also provide the evaluation of runtime
and FLOPs \citep{blalock2020stateneuralnetworkpruning}. 

The runtime in \Cref{tab:Runtime} is obtained by testing on 2 A100 GPUs on the WinoGrande Dataset with a batch size of 64. It is worth noting that the runtime of the merged methods is even slower compared to the original Mixtral model. This is likely due to the code we referred from \cite{li2023merge}, which only creates references from the experts that are merged but does not exactly reduce the number of experts in the model. Note that the sparse matrices induced by unstructured pruning are stored as normal matrices instead of sparse matrices here. We could observe that \alg~does not influence the time complexity while reducing the memory.

For the FLOPs evaluation, as shown in \Cref{tab:FLOPs}, structured pruning and MLP Fusion, which reduce the intermediate dimension of weight matrices, lead to a significant reduction in FLOPs. It is also important to note although unstructured pruning can reduce FLOPs, it does not reduce space storage as ResMoE does as detailed in \Cref{sec:memory}. In this regard, ResMoE (UP) and the merge methods maintain similar FLOPs compared to the original Mixtral. ResMoE (SVD) has more FLOPs compared to vanilla SVD due to the extra FLOPs brought in with the center expert, but it still manages to reduce the FLOPs compared to the original model.

\begin{table}[t!]
\caption{Runtime of Mixtral on Winogrande.}
    \centering
    \begin{tabular}{lc}
     \toprule
& Runtime (s) \\
\midrule
Mixtral&	39.44$\pm$0.30 \\
\midrule
UP	&39.01$\pm$0.21 \\
SP	&37.15$\pm$0.22\\
SVD	& 38.96$\pm$0.31\\
M-SMoE	&49.19$\pm$0.12\\
Git Re-Basin	&48.53$\pm$0.09\\
MEO	&49.51$\pm$0.18\\
MLP Fusion	&38.53$\pm$0.26\\
\midrule
ResMoE (UP)	&38.85$\pm$0.28\\
ResMoE (SVD)	&38.12$\pm$0.19\\
\bottomrule
    \end{tabular}
    
    \label{tab:Runtime}
\end{table}

\begin{table}[t!]
\caption{FLOPs evaluation of Mixtral \& DeepSeekMoE.}
    \centering
    \begin{tabular}{lcc}
     \toprule
& Mixtral (TFLOPs) & DeepSeekMoE (GFLOPs) \\
\midrule
Full&	3.26 &670.46\\
\midrule
UP	&1.64 &460.24 \\
SP	&1.64 &460.24\\
SVD & 2.21& 480.52\\
M-SMoE	& 3.26       & 670.46\\
Git Re-Basin	&3.26&670.46\\
MEO	&3.26            &670.46\\
MLP Fusion	&1.64    &460.24\\
\midrule
ResMoE (UP)	&3.26&670.46\\
ResMoE (SVD)	&2.73&526.93\\
\bottomrule
\end{tabular}

\label{tab:FLOPs}
\end{table}

\subsection{Pseudocode of \alg}
\label{app:pseudocode}

In considering of reproducibility, we provide the algorithm to perform ResMoE on a model in \Cref{alg:resmoe}, and the dynamic load inference process in \Cref{alg:resmoe_inference}. Please find the source code at \url{https://github.com/iDEA-iSAIL-Lab-UIUC/ResMoE}.

\begin{algorithm}[h!]
\caption{ResMoE}
\label{alg:resmoe}
\begin{algorithmic}
\State {\bfseries Input:} experts weights $E_1, \ldots, E_n$, sparsity ratio $r$
\State {\bfseries Output:} center expert $C$, compressed residuals $R_1, \ldots, R_n$

\State // Step 1: Calculate the center expert using free-support Wasserstein barycenter
\State $C \gets \text{free\_support\_wasserstein\_barycenter}(E_1, \ldots, E_n)$

\State // Step 2: Calculate the residual matrices
\For{$i = 1$ {\bfseries to} $n$}
\State $R_i \gets E_i - C$
\EndFor

\State // Step 3: Apply compression technique to residual matrices with sparsity $r$
\For{$i = 1$ {\bfseries to} $n$}
\State $R_i \gets \text{compress}(R_i, r)$
\EndFor

\State {\bfseries return} $C, R_1, \ldots, R_n$
\end{algorithmic}
\end{algorithm}

\begin{algorithm}[h!]
\caption{Inference}
\label{alg:resmoe_inference}
\begin{algorithmic}
\State {\bfseries Input:} input $x$, center expert $C$, pruned residuals $R_1, \ldots, R_n$, selected experts subscript set $S$
\State {\bfseries Output:} inference result $y$

\State // Step 1: Reconstruct and dynamically load the compressed experts
\State $E_S \gets \emptyset$
\For{$R_i \in R_S$}
\State $E_i \gets C + R_i$
\State $E_S \gets E_S \cup {E_i}$
\EndFor

\State // Step 2: Perform inference using the recovered experts
\State $y \gets \text{perform\_inference}(x, E_S)$

\State {\bfseries return} $y$
\end{algorithmic}
\end{algorithm}

\section{Miscellanies}

\subsection{Adaptability with Expert Parallelism and Tensor Parallelism}
\label{app:parallel}

While the primary focus of this paper is on the inference stage and not on saving memory usage during training, we acknowledge that expert parallelism is an important consideration for the scalability and efficiency of MoE models. 

One feasible approach is to assign different center experts to each GPU, allowing each center expert to handle the experts on its respective GPU during inference. This extension of ResMoE could potentially improve the model's performance by capturing more diverse and complex patterns in the data.

ResMoE is also compatible with tensor parallelism. As shown in \Cref{eqn:mlp_sum}, the output of an MLP in an MoE model can be expressed as the summation of several sub-MLPs. In the context of ResMoE, we can view each sub-MLP as the combination of a center expert and a compressed residual matrix. To utilize Megatron tensor sharding \citep{shoeybi2020megatronlm}, we can partition the center expert and compressed residual matrices into different chunks, with corresponding index ranges residing on different GPUs.

During inference, the input data would be parallelized and passed to the appropriate center expert and residual matrix chunks on each GPU. The partial results from each GPU would then be combined to obtain the final output. This parallelization strategy aligns well with the Megatron tensor sharding approach, which aims to distribute the computation across multiple GPUs to improve efficiency and scalability.

\subsection{Illustration of Model Fusion}
\label{app:modelfusion}

In the work of OT Fusion~\cite{singh2020model}, for the first layer in an two-layer MLP (for example), their algorithm takes the weights $(\mtx W_k^{(1)}, \mtx b_k^{(1)})$ in each expert $E_k$ and in $E_\omega$
as the source distributions and the target distribution, respectively.
They directly take $\mtx W_k^{(1)}$ as the design points (empirical distributions) and then regard their free-support Wasserstein barycenter as the common pattern extraction of the first layer in each expert,
returning $\mtx W_\omega^{(1)}$ and the corresponding permutation matrix $\mtx T_k^{(1)}$'s (obtained from the transport matrices) for $\mtx W_k^{(1)}$. 
Accordingly, they then pre-align the second layers as $\mtx W_k^{(2)} \mtx T_k^{(1)}$
, repeat the procedure above and similarly obtain the extracted layer $\mtx W_\omega^{(2)}$ and the permutation matrices $\mtx T_k^{(2)}$. To recover the $k$-th expert $E_k$, their barycenter expert needs to be transformed as $(\mtx T_k^{(2)})^\T E_\omega(x)$, to fix the order of the output elements. We remark this, along with pre-alignment $\mtx W_k^{(2)} \mtx T_k^{(1)}$ bring overhead during algorithm performing stage.
This is supported by the process of applying OT fusion to calculate the barycenter expert on Mixtral, which takes more than 4 days to complete the whole process, while \alg only takes less than a day. A possible explanation is that Mixtral's MLP consists of three layers, so the layer-by-layer strategy costs about three times more than our method.

\subsection{The MLP Form for Mixtral and DeepSeekMoE}
\label{sec:mix_deep}
The output of an MLP in Mixtral and DeepseekMoE can be rewritten into:
\begin{small}
\begin{align*}
E_{k}(\mtx x) 
&= \mtx W_{k}^{(2)}\cdot\operatorname{SwiGLU}(\mtx x) +\mtx b^{(2)}_k \\
&= \mtx W_{k}^{(2)} \cdot \left[\sigma\left(\mtx W_{k}^{(1)}\cdot \mtx x+\mtx b^{(1)}_k\right)\odot\left(\mtx W_{k}^{(3)} \cdot \mtx x+\mtx b^{(3)}_k\right)\right] +\mtx b^{(2)}_k\\
&= \sum_{i=1}^{p_{I}} \mtx W^{(2)}_{k,\cdot,i}\cdot\left[\sigma\left(\left \langle \mtx W_{k,i,\cdot}^{(1)},\mtx x \right \rangle +\mtx b_{k, i}^{(1)} \right)\cdot \left(\left\langle \mtx W^{(3)}_{k,i,\cdot},\mtx x \right \rangle +\mtx b_{k, i}^{(3)}\right) \right]+\mtx b^{(2)}_k,
\end{align*}
\end{small}
where $\sigma(\mtx x)=\operatorname{Swish}_{\beta}(\mtx x)=\mtx x \sigma(\beta \mtx x)=\dfrac{\mtx x}{1+e^{-\beta \mtx x}}$, with $\beta$ setting to $1$.

Similarly, for Mixtral and DeepSeekMoE, the extraction objective is:
\begin{align*}
    \min_{\substack{\mtx W_\omega^{(1)}, \mtx b_\omega^{(1)}, \mtx W_\omega^{(3)}, \mtx b_\omega^{(3)}, \mtx W_\omega^{(2)}\\  \mtx T_k \in \m P, k \in [N]}}
    &\frac{1}{N}\sum_{k=1}^N \bigg[\Big\|\mtx T_k \paren{\mtx W_k^{(1)}, \mtx b_k^{(1)}, \mtx W_k^{(3)}, \mtx b_k^{(3)}} \\
    &\qquad - \paren{\mtx W_\omega^{(1)}, \mtx b_\omega^{(1)}, \mtx W_\omega^{(3)}, \mtx b_\omega^{(3)}} \Big\|_F^2 \\
    &\qquad + \left\|\mtx W_k^{(2)} \mtx T_k^\T - \mtx W_\omega^{(2)} \right\|_F^2\bigg].
\end{align*}

Then we can extend the $\mtx W_k$ to:
\begin{align*}
    \mtx W_k = \brkt{\mtx W_k^{(1)}, \mtx b_k^{(1)},\mtx W_k^{(3)},\mtx b_k^{(3)}, (\mtx W_k^{(2)})^\T},
\end{align*}
and \alg~can then be similary applied to Mixtral and DeepSeekMoE.

\section{Proof of Proposition~\ref{prop:equivalence}}
\label{sec:proof_equivalence}

For the reader's convenience, we recall \Cref{prop:equivalence} as follows.
\equivalence*
\begin{proof}

We recall $\mu_i,\mu_{\omega}$ are uniformly distributed over the $p_\text{I}$ rows of $\mtx W_{i}$ and $\mtx W_{\omega}$, respectively. $\text{OT}\left(\mu_i, \mu_\omega(\mtx W_\omega)\right)$ is the optimal transport matrix $\mtx M$ from $\mu_i$ to $\mu_\omega$ of the following problem:
\begin{equation}
\label{prob:pr_ot}
\text{OT}(\mu,\nu)\defeq\underset{\mtx M\in U(\alpha,\beta)}{\text{argmin}}\left \langle \mtx M,\mtx C\right \rangle,
\end{equation}        
where $\mtx C=\left[ \|\mtx W_{k_{i}} - \mtx W_{\omega_{j}})\|^{2} \right]_{ij}\in \mathbb{R}^{p_{{I}}\times p_{{I}}}, U(\frac{\mathbf{1}_{p_{{I}}}}{p_{{I}}},\frac{\mathbf{1}_{p_{{I}}}}{p_{{I}}})\defeq\{\mtx M\in  \mathbb{R}_{+}^{p_{{I}}\times p_{{I}}} \mid \mtx M\mathbf{1}_{p_{{I}}}=\frac{\mathbf{1}_{p_{{I}}}}{p_{{I}}},\mtx M^{T}\mathbf{1}_{p_{{I}}}=\frac{\mathbf{1}_{p_{{I}}}}{p_{{I}}}\}$. 

We first relate the transport matrix to the permutation matrices $\mtx T_k$'s.
\citet[Proposition~2.1]{peyr2020computational} shows the optimal solution to problem~\eqref{prob:pr_ot} is exactly a permutation matrix, up to a constant factor $p_\text{I}$.
Now straightforwardly, problem~\eqref{eqn:min_frob} can be rewritten as:
\begin{equation}
\label{eq:fin}
    \min_{\substack{\mtx W_\omega \\ \mtx  T_k \in  P, k \in [N]}}
    \frac{1}{N}\sum_{k=1}^N \bigg[\left\| \mtx  T_k {\mtx  W_k} - {\mtx W_\omega} \right\|_F^2\bigg],
\end{equation}
where $\mtx W_k = [\mtx W_k^{(1)},\mtx b_k^{(1)},(\mtx W_k^{(2)})^T] \in \mathbb{R}^{p_{{I}}\times (2p+1)}$, and $\mtx W_\omega = [\mtx W_\omega^{(1)},\mtx b_\omega^{(1)},(\mtx W_\omega^{(2)})^T]\in \mathbb{R}^{p_{{I}}\times (2p+1)}$.

We denote the objective function in problem~\eqref{eq:fin} as $f(\mtx W_\omega; \set{\mtx T_k}_{k=1}^N ) = \sum_{k=1}^N \frac{1}{N}[\left\| \mtx  T_k {\mtx  W_k} - {\mtx W_\omega} \right\|_F^2]$,
and take $\mtx W^*_\omega$ as the optimal solution to the Wasserstein barycenter problem~\eqref{eq:wb}.
For the given $\mtx W^*_\omega$, we further denote $\mtx T_k^*\defeq\underset{\mtx T_k}{\text{argmin}}f(\mtx W_w^*;\mtx T_k), \forall k \in [N]$. 
The rest of the proof is to show $f(\mtx W^*_\omega; \set{\mtx T_k^*}_{k=1}^N ) = $ \Cref{eqn:min_frob}.

\circled{1} We start with the first side: \Cref{eqn:min_frob} $\leq f(\mtx W^*_\omega; \set{\mtx T_k^*}_{k=1}^N )$. 
We indeed immediately have:
\begin{align*}
    \eqref{eqn:min_frob} = \underset{\mtx W_\omega, \mtx T_k}{\min}f(\mtx W_\omega;\set{\mtx T_k}_{k=1}^N ) \leq f(\mtx W^*_\omega; \set{\mtx T_k^*}_{k=1}^N ),
\end{align*}
due to the definition of $\min$ in \Cref{eqn:min_frob}.

\circled{2} For the other direction, we first show the barycenter loss $\eqref{eq:wb} \leq \eqref{eqn:min_frob}$.
Through the definition of $W_2$ distance, we have
\begin{align*}
    W_{2}^{2}(\mu_i, \mu_\omega( \mtx W_\omega))&\leq \left\| \mtx  T_k {\mtx  W_k} - {\mtx W_\omega} \right\|_F^2,\: \forall \mtx T_k, \mtx W_\omega\\
\Rightarrow
    \frac{1}{N}\sum_{k=1}^{N} W_{2}^{2}(\mu_i, \mu_\omega( \mtx W_\omega))&\leq\frac{1}{N}\sum_{k=1}^{N} \left\| \mtx  T_k {\mtx  W_k} - {\mtx W_\omega} \right\|_F^2,\: \forall \mtx T_k, \mtx W_\omega\\
\Rightarrow
    \frac{1}{N} \sum_{k=1}^{N} W_{2}^{2}(\mu_i, \mu_\omega( \mtx W_\omega)) &\leq \frac{1}{N} \sum_{k=1}^{N} \underset{\mtx T_k}{\min} \left\| \mtx  T_k {\mtx  W_k} - {\mtx W_\omega} \right\|_F^2,\: \forall \mtx W_\omega.
\end{align*}
The inequality will still hold when we minimize the two sides both over $\mtx W_\omega$:
\begin{align*}
\eqref{eq:wb} &= \underset{\mtx W_\omega}{\min} \frac{1}{N}\sum_{k=1}^{N} W_{2}^{2}(\mu_i, \mu_\omega( \mtx W_\omega)) \\
&\leq \underset{\mtx W_\omega}{\min}\frac{1}{N} \sum_{k=1}^{N} \underset{\mtx T_k}{\min} \left\| \mtx  T_k {\mtx  W_k} - {\mtx W_\omega} \right\|_F^2 \\
&= \eqref{eqn:min_frob}.
\end{align*}
To close the proof, it suffices to show that $f(\mtx W_\omega^*;\mtx T_k^*)=\eqref{eq:wb}$.
We show the equivalence as follows:
\begin{align*}
    f(\mtx W_\omega^*;\mtx T_k^*) &= \frac{1}{N}\sum_{k=1}^N \left\| \mtx  T_k^* {\mtx  W_k} - {\mtx W_\omega^*} \right\|_F^2 \\
    &= \frac{1}{N}\sum_{k=1}^N \underset{\mtx T_k}{\min} [\left\| \mtx  T_k {\mtx  W_k} - {\mtx W_\omega^*} \right\|_F^2]\\
    &=\frac{1}{N}\sum_{k=1}^N W_{2}^{2}(\mu_i,\mu_\omega( \mtx W_\omega^*)),
\end{align*}
where the last equation holds again thanks to \citet[Proposition~2.1]{peyr2020computational}.
Using the fact that $W^*_\omega$ is the optimal solution to Wasserstein barycenter problem \eqref{eq:wb}, we finally attain
\begin{align*}
f(\mtx W_\omega^*;\mtx T_k^*) &= \frac{1}{N}\sum_{k=1}^N W_{2}^{2}(\mu_i,\mu_\omega( \mtx W_\omega^*)) \\
&= \underset{\mtx W_\omega}{\min}\dfrac{1}{N}\sum_{i=1}^{N}W_{2}^{2}(\mu_i, \mu_\omega( \mtx W_\omega)) \\
    &= \eqref{eq:wb},
\end{align*}
which completes the proof.
\end{proof}

\end{document}